\documentclass[11pt]{article}
\usepackage{amsfonts,amsmath,amsthm, mathtools}
\usepackage{amssymb}
\usepackage{caption}
\usepackage{authblk}
\usepackage{subcaption}
\usepackage{verbatim}
\usepackage{mathrsfs}
\usepackage{algorithm}
\usepackage{algorithmic}
\usepackage{comment}
\usepackage{hyperref}
\usepackage{multirow, tabularx, booktabs}
\usepackage[dvipsnames]{xcolor}

\usepackage{bbold}
\usepackage{fullpage}
\usepackage{xspace}

\numberwithin{equation}{section}

\theoremstyle{plain}
\newtheorem{thm}{Theorem}[section]

\newtheorem{lemma}{Lemma}[section]
\newtheorem*{theorem*}{Theorem}

\theoremstyle{definition}

\newtheorem{rem}{Remark}[section]
\usepackage{graphicx}\usepackage[margin=1in]{geometry}

\newcommand{\rr}{\mathbb{R}}

\newcommand{\hp}{\mathbb{H}}

\newcommand{\pp}{\partial}
\newcommand{\upa}{\uparrow}
\newcommand{\dn}{\downarrow}

\newcommand{\dd}{\mathbb{D}}

\providecommand{\keywords}[1]
{
  \small	
  \textbf{\textit{Keywords---}} #1
}

\begin{document}
\title{Consistent Spectral Clustering in Hyperbolic Spaces}
\author[1]{Sagar Ghosh}
\author[2]{Swagatam Das}

\affil[1]{Computer and Communication Sciences Division, Indian Statistical Institute, Kolkata}
\affil[2]{Electronics and Communication Sciences Unit, Indian Statistical Institute, Kolkata}
\maketitle

\begin{abstract}
    Clustering, as an unsupervised technique, plays a pivotal role in various data analysis applications. Among clustering algorithms, Spectral Clustering on Euclidean Spaces has been extensively studied. However, with the rapid evolution of data complexity, Euclidean Space is proving to be inefficient for representing and learning algorithms. Although Deep Neural Networks on hyperbolic spaces have gained recent traction, clustering algorithms or non-deep machine learning models on non-Euclidean Spaces remain underexplored. In this paper, we propose a spectral clustering algorithm on Hyperbolic Spaces to address this gap. Hyperbolic Spaces offer advantages in representing complex data structures like hierarchical and tree-like structures, which cannot be embedded efficiently in Euclidean Spaces. Our proposed algorithm replaces the Euclidean Similarity Matrix with an appropriate Hyperbolic Similarity Matrix, demonstrating improved efficiency compared to clustering in Euclidean Spaces. Our contributions include the development of the spectral clustering algorithm on Hyperbolic Spaces and the proof of its weak consistency. We show that our algorithm converges at least as fast as Spectral Clustering on Euclidean Spaces. To illustrate the efficacy of our approach, we present experimental results on the Wisconsin Breast Cancer Dataset, highlighting the superior performance of Hyperbolic Spectral Clustering over its Euclidean counterpart. This work opens up avenues for utilizing non-Euclidean Spaces in clustering algorithms, offering new perspectives for handling complex data structures and improving clustering efficiency.
\end{abstract}
\section{Introduction}

In the realm of machine learning, the pivotal process of categorizing data points into cohesive groups remains vital for uncovering patterns, extracting insights, and facilitating various applications such as customer segmentation, anomaly detection, and image understanding \cite{survey_1}. Among the paradigms of clustering algorithms \cite{fil}, alongside \textit{Partitional, Hierarchical,} and \textit{Density-based} techniques, \textit{Spectral Clustering} on Euclidean spaces has garnered extensive research attention \cite{tutorial}. Spectral clustering operates in the spectral domain, utilizing the eigenvalues and eigenvectors of the Laplacian of a similarity graph constructed from the data. This algorithm initially constructs a similarity graph, where nodes represent data points and edges indicate pairwise similarities or affinities among the data points. It then computes the graph Laplacian matrix, capturing the graph's structural properties and encoding relationships among the data points. Since its inception [See \cite{hoff} and \cite{fied}], the Euclidean version of Spectral Clustering has significantly evolved. For simple clustering with two labels, this method considers the eigenvector corresponding to the second smallest eigenvalue of a specific graph Laplacian constructed from the affinity matrix based on the spatial position of the sample data. It then performs the $2$-means clustering on the rows of the Eigenmatrix [whose columns consist of the eigenvectors corresponding to the smallest two eigenvalues], treating the rows as sample data points, and finally returns the cluster labels to the original dataset. This particular form of spectral clustering finds applications \cite{swamy} in Speech Separation \cite{bach}, Image Segmentation \cite{tung}, Text Mining \cite{dhillon}, VLSI design \cite{kahng}, and more. A comprehensive tutorial is also available at \cite{tutorial}. Here, we will briefly review the most commonly used Euclidean Spectral Clustering Algorithm.

When connectedness is a crucial criterion for the clustering algorithms, the conventional form of Spectral Clustering emerges as a highly effective approach. It transforms the standard data clustering problem in a given Euclidean space into a graph partitioning problem by representing each data point as a node in the graph. Subsequently, it determines the dataset labels by discerning the spectrum of the graph. Beginning with a set of data points $X:={x_1,x_2,...,x_N}\in\rr^l$. A symmetric similarity function (also known as the kernel function) $k_{i,j}:=k(x_i,x_j)$, along with the number of clusters $k$, we construct the similarity matrix $W:=w(i,j)=k_{ij}$. In its simplest form, Spectral Clustering treats the similarity matrix $W$ as a graph and aims to bipartite the graph to minimize the sum of weights across the edges of the two partitions. Mathematically, we try to solve an optimization problem \cite{tutorial} of the following form:
\begin{align} \label{eq:1} 
    \min_{U\in\rr^{N\times l}} C:=\min_{U\in\rr^{N\times l}} Tr(U^tL^\prime U) \hspace{0.5ex} \text{s.t. } \hspace{0.5ex} U^tU=I,
\end{align} where  $L:=D-W$ is the Graph Laplacian and the degree matrix  $D:=diag(d_1,d_2,,,,d_N), d_i:=\sum_{j=1}^Nw(i,j)$ and $L^\prime:=D^{-1/2}LD^{-1/2}=I-D^{-1/2}WD^{-1/2}$. $U$ is a label feature matrix, $l$ being the number of label features. Then, the simplest form of spectral clustering aims to minimize the trace by the feature matrix $U$ by considering the first $l$ eigenvectors of $L^\prime$ as its rows.

In order to extend the algorithm to form $k (>2)$ clusters, we consider the  $k$ eigenvectors $[u_1,u_2,...,u_k]\in\rr^N$ corresponding to the $k$ smallest eigenvalues of $\Tilde{L}$. Then, treating the $N$ rows of the Eigenmatrix $U:=[u_1,u_2,...,u_k]$ as sample data points in $\rr^k$, we perform the $k-$means clustering and label them. Finally, we return the labels to the original data points in the order they were taken to construct the degree matrix $W$. Some of the variants of this algorithm can be found at \cite{tutorial}. 

\begin{figure*}[ht]
    \centering
    \begin{subfigure}[b]{0.3\textwidth}
         \centering
         \includegraphics[width=\textwidth]{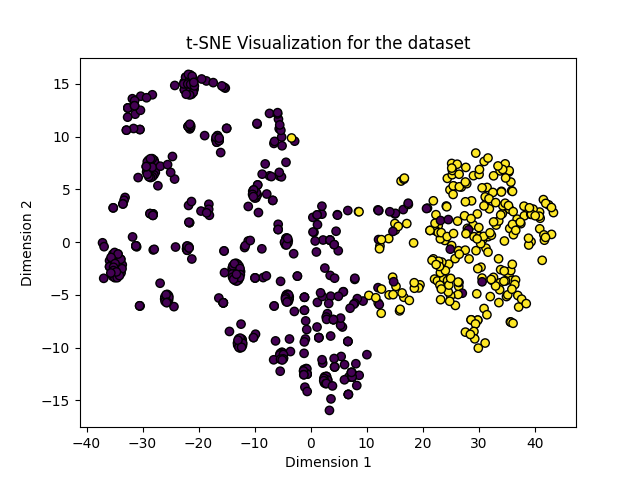}
         \caption{Dataset WBCD}
         \label{fig:dataset_wisc_0}
     \end{subfigure}
     \hfill
      \begin{subfigure}[b]{0.3\textwidth}
         \centering
         \includegraphics[width=\textwidth]{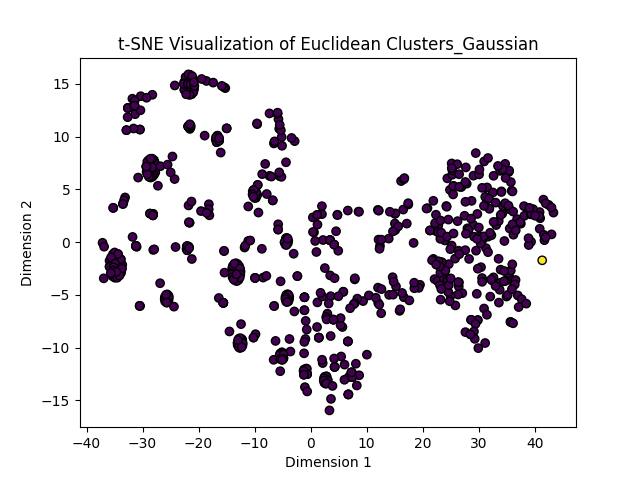}
         \caption{WBCD-ESCA Clusters}
         \label{fig:euclid_wisc_0}
     \end{subfigure}
     \hfill
      \begin{subfigure}[b]{0.3\textwidth}
         \centering
         \includegraphics[width=\textwidth]{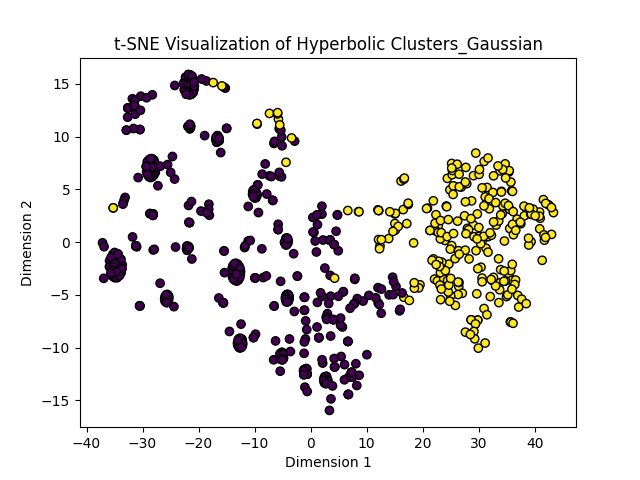}
         \caption{WBCD-HSCA Clusters}
         \label{fig:hyp_wisc_0}
     \end{subfigure}
\caption{Consider the Wisconsin Breast Cancer Dataset taken from the UCL Machine Learning Repository. The leftmost figure describes t-SNE visualization of the clusters of two types of tumors: malignant(yellow dots) and benign(brown dots). Due to the presence of one predominant connected component, Euclidean spectral clustering forms only one cluster (with only one isolated point in the other cluster), whereas hyperbolic spectral clustering forms the clusters after separating two hierarchies present in the data. Clearly, the hyperbolic clusters provide much more accurate description of the dataset over its Euclidean counterpart.} \label{fig:10}
\end{figure*}

This form of Euclidean Spectral Clustering works well on Euclidean datasets, which are relatively simple and have some distinct connectedness. But as data evolves rapidly and takes extremely complex shapes, Euclidean space becomes a more inefficient model space in capturing the intricate patterns in the data and thereby helping learning algorithms in that set-up. Thus, visualizing data from a more complex topological perspective and developing unsupervised learning techniques on these complex shapes can have a consequential impact on revealing the true nature of the data. Although Deep Neural Networks on hyperbolic spaces have recently gained some attention and found applications in various computer vision problems [see \cite{peng}, \cite{ganea}, \cite{chen}, \cite{chami}], clustering algorithms or non-deep machine learning models on non-Euclidean Spaces are one of the least explored domains. In this paper, we scrutinize this domain by proposing a spectral clustering algorithm on hyperbolic spaces. While Euclidean Space fails to provide meaningful information and lacks its representational power in implementing arbitrary tree/graph-like structures or hierarchical data,  that can not be embedded in Euclidean Spaces with arbitrary low dimensions, and at the same time, these spaces do not preserve the associated metric (similarity measure/dissimilarity measure), hyperbolic spaces can be beneficial to implement the structure obtained from these examples [see \cite{miao}, \cite{mondal}] and even some hyperbolic spaces with shallow dimension can be very powerful to represent these data and clustering on these spaces can result in far better efficiency compared to the clustering in Euclidean Spaces like Figure \ref{fig:10}. Our main contributions to this paper are:
\begin{itemize}
    \item We propose the spectral clustering algorithm on hyperbolic spaces in which an appropriate hyperbolic similarity matrix replaces the Euclidean similarity matrix.
    \item We also provide a theoretical analysis concerning the weak consistency of the algorithm and prove that it converges (in the sense of distribution) at least as fast as the spectral clustering on Euclidean spaces. 
    \item Advancing a step further, we present the hyperbolic versions of some of the well-known variants of Euclidean spectral clustering, such as Fast Spectral Clustering with approximate eigenvectors (FastESC) \cite{he}  and the Approximate Spectral Clustering with $k-$ means-based Landmark Selection \cite{chencai}. 
\end{itemize}

Having said that, we organize the rest of our paper in the following way. In Section \ref{sec:2}, we will give a brief overview of the works related to the Euclidean Spectral Clustering and its variants. We will also discuss why we need to consider a hyperbolic version of the Euclidean Spectral Clustering and some of its variants. Section \ref{sec:3} lays out the mathematical backgrounds pertaining to our proposed algorithm. We will discuss several results which will enable us to formulate the algorithm in a rigorous manner. We give the details of our proposed algorithm in Section \ref{sec:4}. We discuss the motivation behind the steps related to our algorithm. Section \ref{sec:5} has been dedicated for the proof of the weak consistency of the proposed algorithm. Section \ref{sec:6} presents and discusses the experimental results. Finally, conclusions are drawn in Section \ref{sec:7}.

\section{Related Works}\label{sec:2}

We commence with a brief overview of prominent variants of Euclidean spectral clustering:

\begin{enumerate}
\item \textbf{Bipartite Spectral Clustering on Graphs (ESCG)}: Introduced by Liu et al. \cite{jliu} in 2013, this algorithm primarily aims to reduce the time complexity during spectral decomposition of the affinity matrix by appropriately transforming the input similarity matrix of a Graph dataset. The method involves randomly selecting $d (\ll n)$ seeds from a given Graph of input size $n$, followed by generating $d$ supernodes using Dijkstra's Algorithm to find the shortest distance from the Graph nodes to the seeds. This process reduces the size of the similarity matrix $\Tilde{W}:=RW$, where $R$ is the indicator matrix of size $d \times n$ and $W$ is the original affinity matrix. Subsequently, it proceeds with spectral decomposition of the normalized $Z:=D_2^{-1/2}\Tilde{W}D_1^{-1/2}$, where $D_1$ and $D_2$ are diagonal matrices containing the column and row sums of $\Tilde{W}$, respectively. The algorithm computes the $k$ largest eigenvectors of $ZZ^t$ and generates $k$ clusters based on the $k$-means algorithm on the matrix $U:=D_1^{-1/2}X$, where $X$ is the right singular matrix in the singular value decomposition of $Z$.
\item \textbf{Fast Spectral Clustering with approximate eigenvectors (FastESC)}: Developed by He et al. \cite{he} in 2019, this algorithm initially performs $k$-means clustering on the dataset with a number of clusters greater than the true clusters and then conducts spectral clustering on the centroids obtained from the $k$-means. Similar to ESCG, this algorithm also focuses on reducing the size of the input similarity matrix for spectral clustering.

\item \textbf{Low Rank Representation Clustering (LRR)}: Assuming a lower-rank representation of the dataset $X:=[x_1,x_2,...,x_n]$, where each $x_i$ is the $i$-th data vector in $\mathbb{R}^D$, this algorithm aims to solve an optimization problem to minimize the rank of a matrix $Z$ subject to $X=AZ$. Here, $A=[a_1,a_2,...,a_m]$ is a dictionary, and $Z:=[z_1,z_2,...,z_n]$ is the coefficient matrix representing $x_i$ in a lower-dimensional subspace. The algorithm iteratively updates $Z$ and an error matrix $E$ as proposed by Liu et al. in \cite{liu}.

\end{enumerate}

Among other variants of Spectral Clustering, such as Ultra-Scalable Spectral Clustering Algorithm (U-SPEC) \cite{wulai} or Constrained Laplacian Rank Clustering (CLR) \cite{jordan}, the primary objective remains consistent - to enhance efficiency by reducing the burden of the spectral decomposition step through minimizing the size of the input similarity matrix. However, there has been minimal exploration regarding the translation of these algorithms into a hyperbolic setup. In this context, this marks the initial attempt to elevate non-deep machine learning algorithms beyond Euclidean Spaces.

Consider a dataset with a hierarchical structure, represented as a top-to-bottom tree. As we traverse along nodes away from the root, the distance among nodes at the same depth but with different parents grows exponentially with respect to their heights. This exponential growth renders Euclidean Space unsuitable as a model space for representing the hierarchy. Hyperbolic spaces are better suited for this purpose, where distance grows exponentially as we move away from the origin/root node. Some of these issues have been addressed in a series of papers \cite{peng}, \cite{ganea}, \cite{chami} in the context of Deep Neural Networks. Recently in context to downstream self-supervised tasks, the authors in \cite {long} presented scalable Hyperbolic Hierarchical Clustering (sHHC), which learns continuous hierarchies in hyperbolic space. Leveraging these hierarchies from sound and vision data, hierarchical pseudo-labels are constructed, leading to competitive performance in activity recognition compared to recent self-supervised learning models. In this work, we propose a general-purpose spectral clustering algorithm on a chosen hyperbolic space after embedding the original Euclidean dataset into the space in a particular way that minimally perturbs the inherent hierarchy. %After implementing the proposed algorithm in Section 4, we observed significantly improved evaluation metrics compared to conventional Euclidean Spectral Clustering or some of its other variants on datasets with inherent hierarchy. Before delving straight into the algorithm, we review some definitions and results from Differential Geometry and Riemannian Manifolds. 

\section{Mathematical Preliminaries} 
For detailed exposure to the theoretical treatment, refer to \cite{carmo}, \cite{tu}, and \cite{lang}.
\label{sec:3}
\subsection{Smooth Manifold}
\subsubsection*{\textbf{Manifolds:}} A topological space $M$ is locally Euclidean of dimension n, if for every $p\in M$, there exists a neighbourhood $U$ and a map $\phi$ such that $\phi:U\to \phi(U)$, $\phi(U)  $ being open in $\rr^n$, is a homeomorphism. We call such a pair ($U,\phi:U\to \phi(U)$) a chart or co-ordinate system on $U$. An $n$ dimensional Topological Manifold is a Hausdorff, Second Countable and Locally Euclidean of dimension $n$ [see \cite{tu}]. 

\subsubsection*{\textbf{Charts and Atlas:}} We say two charts $(U,\phi)$ and $(V,\psi)$ to be $C^{\infty}$ compatible if and only if $\phi\circ{\psi}^{-1}:\psi(U\cup V)\to \phi(U\cup V)$ and $\psi\circ\ {\phi}^{-1}: \phi(U\cup V)\to \psi(U\cup V)$ are $C^{\infty}$ maps between Euclidean Spaces. A $C^\infty$ Atlas or simply an Atlas on a Euclidean Space $M$ is a collection $\mathcal{A}:=\{(U_{\alpha}, \phi_{\alpha})\}_{\alpha\in I}$ such that $U_{\alpha}$'s are pairwise $C^\infty$ compatible and $M=\cup_{\alpha\in I}U_\alpha$. It is worth noting that any Atlas is Contained in a unique Maximal Atlas. 

\subsubsection*{\textbf{Smooth Manifold:}} A smooth or $C^\infty$ manifold $M$ is a topological manifold equipped with a maximal atlas. Throughout this article, we will exclusively refer to a smooth manifold when discussing manifolds.

%Now, let's delve into the concept of tangent space on manifolds. 
Consider a point $p\in \mathbb{R}^n$. If two functions coincide on certain neighborhoods around $p$, their directional derivatives will also be identical. This allows us to establish an equivalence relation among all $C^\infty$ functions within some neighborhood of $p$. Let $U$ be a neighborhood of $p$, and consider all $f:U\to \mathbb{R}$ where $f$ is $C^\infty$. We define two functions $(f,U)$ and $(g,V)$ as equivalent if there exists a neighborhood $W\subseteq U\cap V$ such that $f=g$ when restricted to $W$. This equivalence relation holds for all $C^\infty$ functions at $p$, forming what is known as the set of \textbf{germs of $C^\infty$ functions on $\mathbb{R}^n$ at $p$}, denoted as $C_p^{\infty}(\mathbb{R}^n)$. Similarly, this notion extends to manifolds through charts, where $C_p^{\infty}(M)$ represents the germs of all $C^\infty$ functions on $M$ at $p$.

\subsubsection*{\textbf{Derivation and Tangent Space:}} A Derivation $\dd$ on $M$ is a linear map:
\begin{align*}
    \dd:C_p^{\infty}(M)\to \rr,
\end{align*} satisfying the Libneiz's Rule:
\begin{align*}
    \dd(fg) = \dd(f)g(p)+f(p)\dd(g).
\end{align*} A tangent vector at $p$ is a Derivation at $p$. The vector space of all tangent vectors at $p$ is denoted as $T_p(M)$. 

Let $(U,\phi)$ be a chart around $p$. Let $\phi = (x^1, x^2, ...,x^n)$. Let $r^1, r^2, ..., r^n$ be the standard coordinates on $\rr^n$. Then $x^i = r^i\circ \phi:U\to \rr$. If $f\in C_p^{\infty}(M)$, then $\frac{\pp}{\pp x^i}|_p(f):= \frac{\pp}{\pp r^i}|_{\phi(p)}(f\circ \phi^{-1})$. If $v\in T_p(M)$, $v = \sum_{i=1}^{n}a_i\frac{\pp}{\pp x^i}$.

\subsubsection*{\textbf{Vector Fields:}} A Vector Field on an open subset $U$ of $M$ is a function that assigns to each point $p\in U$ to a tangent vector $X_p$ in $T_p(M)$. Since $\{\frac{\pp}{\pp x^i}|_p\}_{1\leq i\leq n}$ form a local basis of $T_p(M)$, $X_p = \sum_{i=1}^{n}a_i(p)\frac{\pp}{\pp x^i}$.

\subsubsection*{\textbf{Cotangent Space and Differential Forms:}} Let us also talk briefly about the dual of Tangent Space, also known as the Cotangent Space. It consists of the space of all covectors or linear functionals on the Tangent Space $T_p(M)$ and is denoted by $T^*_p(M)$. Similarly, here we can define Covector Fields or Differential 1-Form on an open set $U$ of $M$ that assigns to each point $p$, a covector $\omega_p\in T^*_p(M)$. Staying consistent with the notation, $\{dx^i|_p\}_{1\leq i\leq n}$ forms a local basis of Cotangent Space dual to $\{\frac{\pp}{\pp x^i}|_p\}_{1\leq i\leq n}$, i.e. 
\begin{align*}
    dx^i|_P\left(\left\{\frac{\pp}{\pp x^j}|_p\right\}\right) = \delta^i_j,
\end{align*} where $\delta^i_j = 1$ when $i=j$ and is $0$ otherwise.
\vspace{-1mm}
\subsubsection*{\textbf{Tensor Product:}} If $f$ is a $k$-linear function and $g$ is a $l$-linear function on a vector space $V$, their tensor product is defined to be the $(k+l)$ linear function $f\otimes g$ as
\begin{align*}
    f\otimes g(v_1, v_2, ..., v_{k+l}) = f(v_1, ..., v_k)g(v_{k+1}, ..., v_{k+l}), 
\end{align*} where $v_1, ..., v_{k+l}$ are arbitrary vectors in $V$.

%Having all these notions in hand and with the help of metirc, now we can outline the Riemanian Manifolds.
\subsection{Riemannian Manifold}
\subsubsection*{\textbf{Riemannian Manifolds:}} A Riemannian Manifold $M$ is a smooth manifold equipped with a smooth Riemannian Metric \cite{carmo}, i.e. a $C^\infty$ map $g:M\to T^*M^{\otimes 2}$, i.e. $\forall p\in M$, $g_p\in T^*_pM^{\otimes 2} $, i.e. $g_p$ is a bilinear functional on $T_p(M)$. [$g_p: T_p(M)\times T_p(M) \to \rr$] If $\phi = (x^1,x^2,..., x^n)$ is a chart on $U\subseteq M$, taking $e_i = \frac{\pp}{\pp x^i}$ and $e^*_i = dx^i$, 
\begin{align*}
    g = \sum_{i,j}g(e_i,e_j)e^*_i\otimes e^*_j.
\end{align*}

Now we state an elementary result from the Riemannian Geometry without proof. 
\subsubsection*{\textbf{Proposition 3.1:} }\label{prop:3.1} Every smooth manifold admits a Riemannian Metric.

\subsubsection*{\textbf{Corollary 1:}} The Usual Spectral Clustering Algorithm on Euclidean Spaces can be applied to Riemannian Manifolds with the Euclidean Metric replaced by Riemannian Metric, at least locally. In particular it can be applied to compact subsets of model hyperbolic spaces with the corresponding hyperbolic metric.

\subsubsection*{\textbf{Length and Distance:}} Let $\gamma:[a,b]\to (M,g)$ be a piecewise smooth curve. The length of $\gamma$, $L(\gamma)$ is defined as, $L(\gamma): = \int_{a}^{b}g_{\gamma(t)}(\gamma^\prime(t), \gamma^\prime(t))^{1/2}$. Given two points $p,q\in M$, we define the distance between $p$ and $q$ as: $d_g(p,q):=\inf\{L(\gamma)$ such that $\gamma$ is a peicewise smooth curve and $\gamma(a)=p, \gamma(b)=q\}$. This distance is also known as the \emph{Geodesic Distance} between $p$ and $q$ on $M$.\\

%Next we will state a proposition without proof, which will be important later. \\

\textbf{Proposition 3.2:}
    Let $M$ be a Riemannian Manifold with Riemannian metric $g$. Then $M$ is a metric space with the distance function $d_g$ defined as above. The metric topology on $M$ coincides with the manifold topology.\\

Let $\chi(M)$ denote the space of all $C^\infty$ vector fields on $M$. Let $X\in \chi(M)$. Then $X:M\to TM\xhookrightarrow{}\rr^n\times \rr^n$. $X(p) = (p,v)$. Considering only second component of a vector field, $X:M\subseteq \rr^n\to\rr^n$. Hence $dX_p:T_pM\to\rr^n$. Given $v\in T_pM$, we define the covariant derivative of $X$ in the direction of $v$ by $\nabla_vX = dX_p(v)^t\in T_pM$. For $Y\in \chi(M)$, $\nabla_YX$ is defined as $(\nabla_YX)(p) = \nabla_{Y(p)}X, p\in M$. This defines a map $\nabla : \chi(M)\times \chi(M)\to \chi(M)$ as
\begin{enumerate}
    \item[(i)] $\nabla_{fX+gY}Z = f\nabla_XZ+g\nabla_YZ, \forall X,Y,Z\in \chi(M)$ and $f,g\in C^\infty(M)$.
    \item[(ii)] $\nabla_X(\alpha Y+\beta Z) = \alpha\nabla_XY+\beta\nabla_XZ, \forall \alpha, \beta\in \rr$ and $X,Y<Z\in \chi(M)$.
    \item[(iii)] $\nabla_{X}(fY) = f\nabla_XY+(Xf)Y, \forall X<Y\in \chi(M)$ and $f\in C^\infty(M)$.
\end{enumerate}
This $\nabla$ is called an induced connection on $M$ and it satisfies two properties:
\begin{enumerate}
    \item (Symmetry) $\nabla_XY-\nabla_YX = [X,Y] = XY-YX$.
    \item (Metric Compatibility) $X<Y,Z> = <\nabla_XY,Z>+<Y,\nabla_XZ>, \forall X,Y,Z\in \chi(M)$.
\end{enumerate}
This $\nabla$ on $(M,g)$ is called the \emph{Levi-Civita Connection} or the \emph{Riemannian Connection}. It can be proved that there exists a unique such connection on a given Riemannian Manifold $(M,g)$ satisfying symmetry and metric compatibility.

\subsubsection*{\textbf{Curvature Operator:}} The curvature operator on $M$ is defined as the $\rr$-trilinear map:\\
$R: \chi(M) \times \chi(M)\times \chi(M) \to \chi(M)$ as 
$(X,Y,Z)\to R(X,Y)Z:= \nabla_Y\nabla_XZ-\nabla_X\nabla_YZ+\nabla_{[X,Y]}Z$.

\subsubsection*{\textbf{Sectional Curvature:}} Let $p\in M$ and $\sigma\subset T_pM$ be a two dimensional subspace. The curvature of $M$ in the direction of $\sigma$ is defined as 
\begin{align}
    \kappa(p,\sigma):=\frac{<R(x,y)x,y>}{|x\wedge y|^2},
\end{align} where 
\begin{align}
    |x\wedge y|^2 = \|x\|^2\|y\|^2-|<x,y>|^2.
\end{align}

\subsubsection*{\textbf{Hyperbolic Spaces:}} An $n$-dimensional hyperbolic space is characterized as the singularly connected, $n$-dimensional complete Riemannian Manifold with a constant sectional curvature of $-1$. Represented as $\mathbb{H}^n$, it stands as the exclusive simply connected $n$-dimensional complete Riemannian Manifold with a constant negative sectional curvature of $-1$. The renowned Killing-Hopf Theorem \cite{lee} affirms that any two such Riemannian Manifolds are isometrically equivalent. While numerous isometric model spaces exist, we will only touch upon four of them briefly. The geodesics on these spaces are illustrated in Figure \ref{fig:Figure 1}.

\begin{enumerate}
    \item \textbf{Poincar\'{e} Half Space Model:}  This is the upper half plane $H:=\{(x_1,x_2,...,x_n)|x_n>0, x_i\in \rr \forall i\in\{1,2,...,n\}\}$. The metric on this space is defined as follows:
suppose $p_1:=(x_1,x_2,...,x_n)$ and $p_2:=(y_1,y_2,...,y_n)$ are two points whose distance we want to compute. Let $\Tilde{p_1}:={x_1,x_2,...,x_{n-1},-x_n}$ be the reflection of $p_1$ with respect to the plane $x_n=0$. Then 
\begin{align}
    d(p_1,p_2)&:=2 \sinh^{-1}\frac{\|p_2-p_1\|}{2\sqrt{x_ny_n}}\\
    &= 2\log\frac{\|p_2-p_1\|+\|p_2-\Tilde{p_1}\|}{2\sqrt{x_ny_n}}
\end{align}

    \item \textbf{Poincar\'{e} Disc Model:} This is a model hyperbolic space in which all points are inside the unit ball in $\rr^n$ and geodesics are either diameters or circular arcs perpendicular to the unit sphere. The metric between two points $p_1$ and $p_2$ ($\|p_1\|,\|p_2\|<1$)is defined as
\begin{align*}
    d(p_1,p_2):=&\\
    &\cosh^{-1}\left(1+\frac{2\|p_2-p_1\|^2}{(1-\|p_1\|^2)(1-\|p_2\|^2)}\right).
\end{align*}

    \item \textbf{Beltrami-Klein Model:} This is a model hyperbolic space in which points are represented by the points in the open unit disc in $\rr^n$ and lines are represented by chords, special type of straight lines with ideal end points on the unit sphere. For $u,v\in B$, the open unit ball, the distance in this model is given as 
    \begin{align}
        d(u,v):=\frac{1}{2}\log\left(\frac{\|aq\|\|pb\|}{\|ap\|\|qb\|}\right),
    \end{align}
    where $a,b$ are ideal points on the boundary sphere of the line $(p,q)$ with $\|aq\|>\|ap\|$ and $\|pb\|>\|bq\|$. 
    \item \textbf{Hyperboloid Model:} This is a model hyperbolic space also known as Minkowski Model which is the forward sheet $S^{+}$ of the two hyperbolic sheets embedded in the $(n+1)$ dimensional Minkowski Space. The hyperbolic distance between two points $u=(x_0,x_1,...,x_n)$ and $v=(y_0,y_1,...,y_n)$ is given as
    \begin{align}
        d(u,v):=\cosh^{-1}(-B(u,v)),
    \end{align} where
    \begin{align*}
        B((x_0,x_1,...,x_n),(y_0,y_1,...,y_n))\\
        =-x_0y_0+\sum_{i=1}^nx_iy_i.
    \end{align*} Note that $B$ is simply the Minkowski Dot Product. 

\end{enumerate}

\subsubsection*{\textbf{Gyrovector Space:}} %In order to study the vector space structure in the context of Hyperbolic Space, the concept of Gyrovector Space was introduced by Abraham A. Ungar [see \cite{ungar}]. This abstraction enables us to define special addition and scalar multiplications based on weekly associative gyrogroups. A sophisticated geometric formalism of these additions and scalar multiplications can be found in Vermeer \cite{vermeer}. 

%We will briefly discuss the M$\ddot{o}$bius Gyrovector Addition and Mobius Scalar Multiplication on the Poincar\'{e} Disc. Since the hyperbolic spaces of a particular dimensions are isometric to each other, the same additive and multiplicative structures for the other model hyperbolic spaces can be obtained  via isometric transformations between these spaces [see \cite{ratcl}]. We will require M$\ddot{o}bius$ addition and multiplication while estimating the intrinsic evaluation metrics such as Davies Boudlin Score or Calinski-Harabasz Index to validate the performance of our proposed algorithm. 

The concept of Gyrovector Space, introduced by Abraham A. Ungar [see \cite{ungar}], serves as a framework for studying vector space structures within Hyperbolic Space. This abstraction allows for the definition of special addition and scalar multiplications based on weakly associative gyrogroups. For a detailed geometric formalism of these operations, Vermeer's work \cite{vermeer} provides an in-depth exploration.

In this context, we will briefly discuss M$\ddot{o}$bius Gyrovector Addition and Mobius Scalar Multiplication on the Poincar'{e} Disc. Due to isometric transformations between hyperbolic spaces of different dimensions, the same additive and multiplicative structures can be obtained for other model hyperbolic spaces [see \cite{ratcl}]. Utilizing M$\ddot{o}$bius addition and multiplication is essential when evaluating intrinsic metrics like the Davies-Bouldin Score or Calinski-Harabasz Index to assess the performance of our proposed algorithm.

\begin{enumerate}
    \item \textbf{M$\ddot{o}$bius Addition:} For two points $u$ and $v$ in the Poincar\'{e} Disc, the M$\ddot{o}$bius addition is defined as:
    $
        u\bigoplus_K v:=
        \frac{(1+2K<u,v>+K\|v\|^2)u+(1-K\|u\|^2)v}{1+2K<u,v>+K^2\|u\|^2\|v\|^2},$
     where $K$ is the curvature and for hyperbolic spaces, $K=-1$.  

    \item \textbf{M$\ddot{o}$bius Scalar Multiplication:} For $r\in\rr$, $c>0$ and $u$ in the Poincar\'{e} Disc, the scalar multiplication is defined as:
    $
        r\bigotimes_c u:= \frac{1}{\sqrt{c}} \tanh\left(r \tanh^{-1}(\sqrt{c}\|u\|)\right)\frac{u}{\|u\|}
    $
    This addition  and scalar multiplication satisfy the axioms pertaining to the Gyrovector Group [see \cite{ungar}].
\end{enumerate}

\begin{figure}
     \centering
     \begin{subfigure}[b]{0.2\textwidth}
         \centering
         \includegraphics[width=\textwidth]{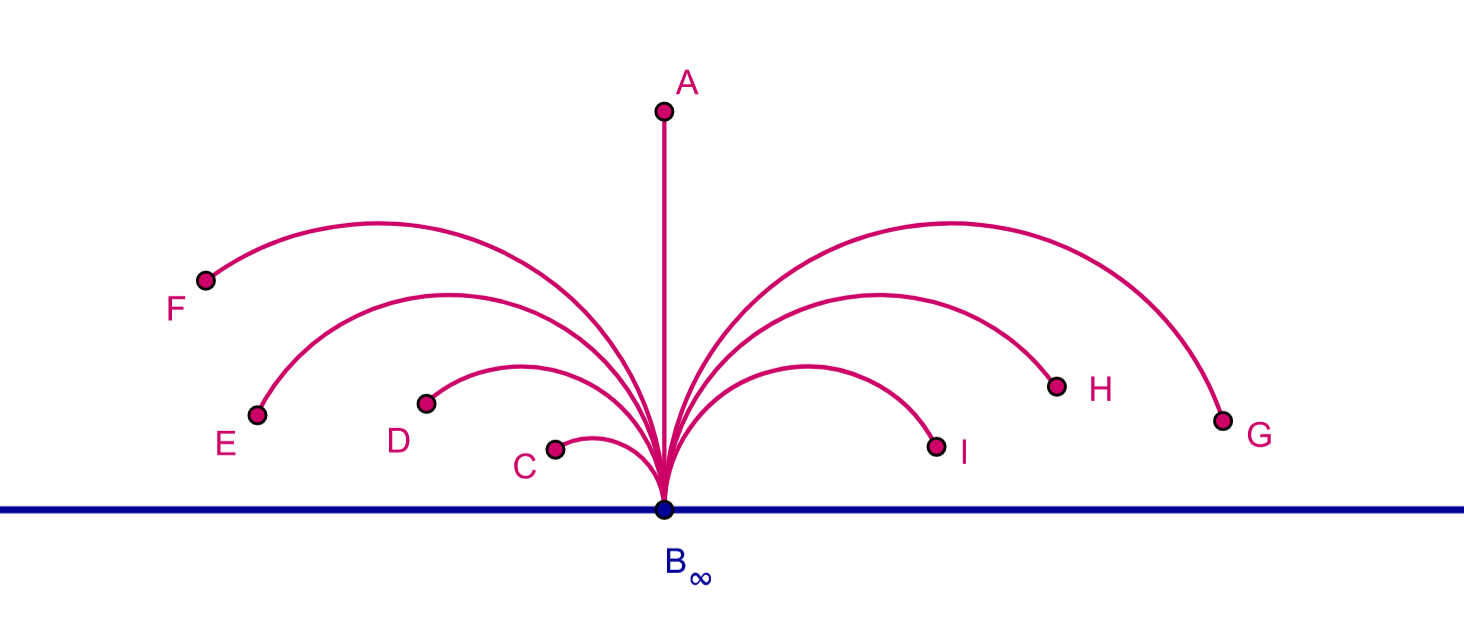}
         \caption{Geodesics in Poincare Half Space Model}
         \label{fig: poincare half space}
     \end{subfigure}
     \hfill
     \begin{subfigure}[b]{0.2\textwidth}
         \centering
         \includegraphics[width=\textwidth]{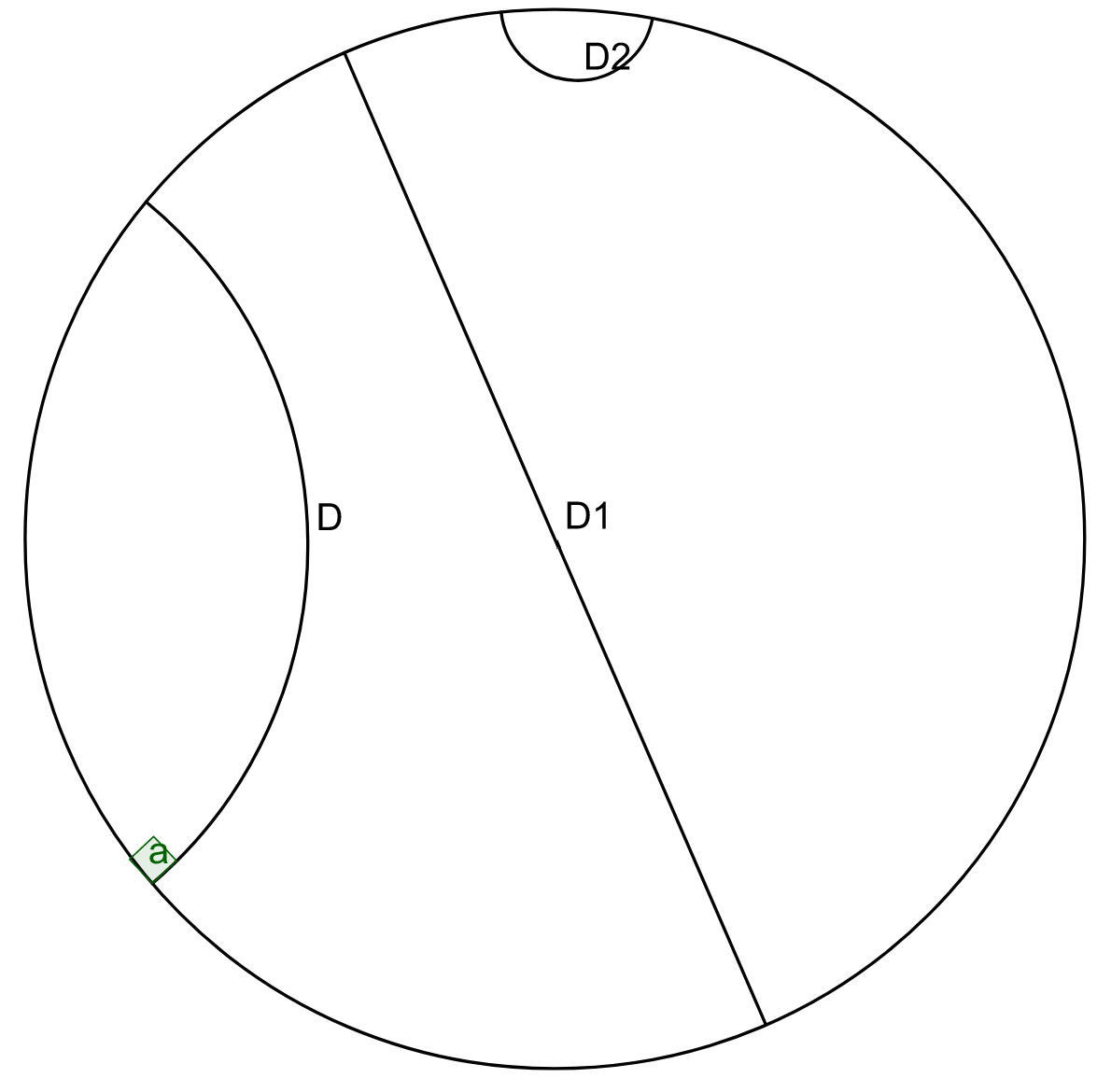}
         \caption{Geodesics in Poincare Disc Model}
         \label{fig: poincare disc}
     \end{subfigure}
     \hfill
     \begin{subfigure}[b]{0.2\textwidth}
         \centering
         \includegraphics[width=\textwidth]{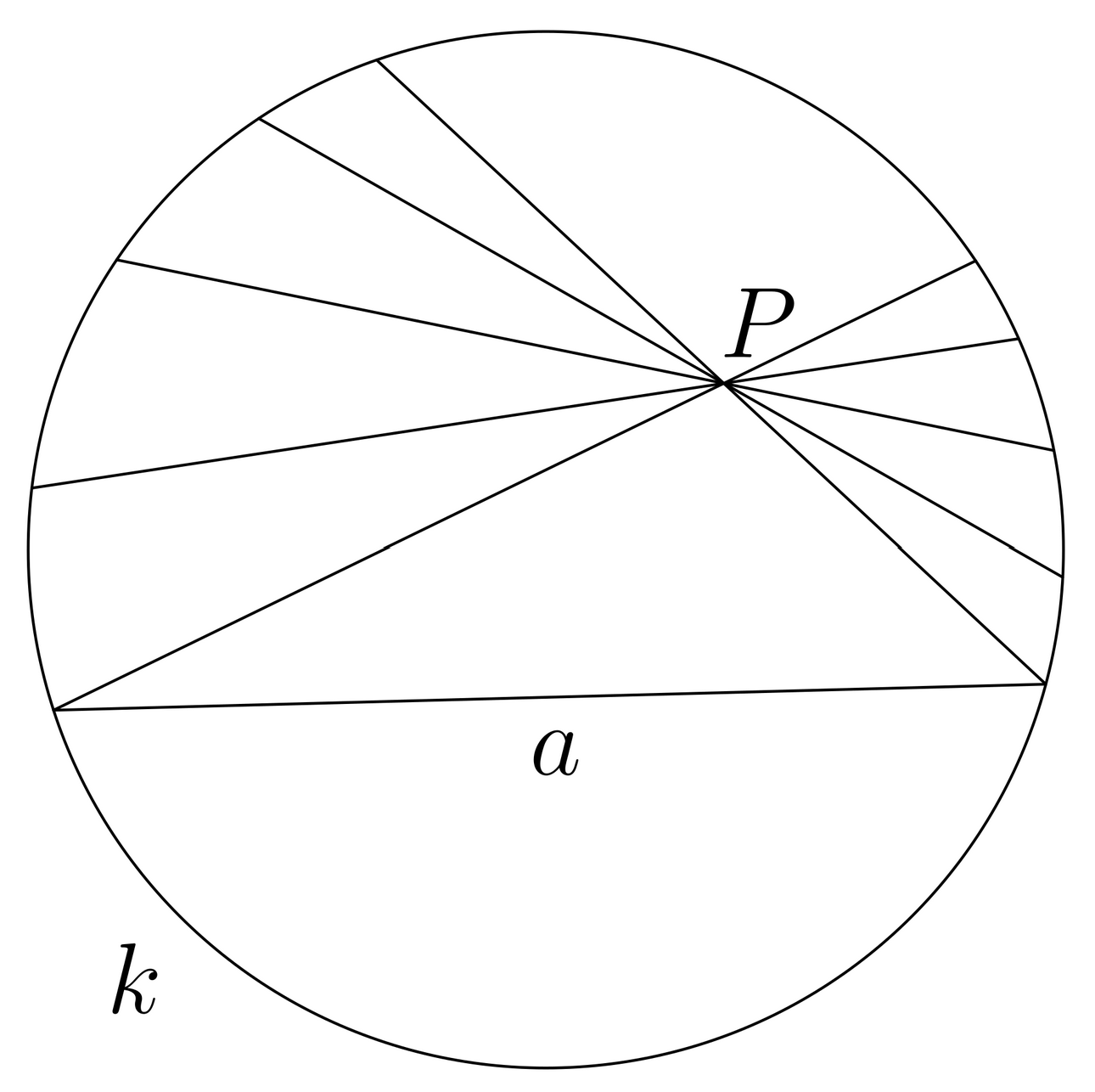}
         \caption{Geodesics in Klein-Beltrami Model}
         \label{fig: Klein-Beltrami Model}
     \end{subfigure}
     \hfill
     \begin{subfigure}[b]{0.2\textwidth}
         \centering
         \includegraphics[width=\textwidth]{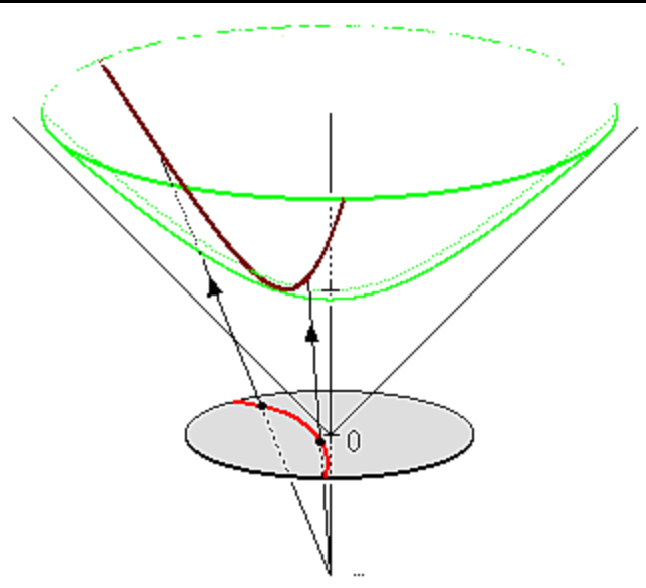}
         \caption{Geodesics in Minkowski Hyperboloid Model}
         \label{fig: Hyperboloid Model}
     \end{subfigure}
     
        \caption{Geodesics in Different Model Hyperbolic Spaces}
        \label{fig:Figure 1}
\end{figure}

\section{Proposed Algorithm}\label{sec:4}
Before discussing the proposed spectral clustering algorithm, we will briefly talk about an equivalent formulation of the Euclidean Spectral Clustering algorithm. In the Introduction \ref{sec:introduction} section, we described the spectral clustering problem as trace minimization problem. In its simplest form, if the two clusters are $A$ and $B$, then solving \ref{eq:1} is equivalent to minimize the normalized cut (N-Cut) problem
\begin{align} \label{eq:2}
    Ncut(A,B):=\frac{Cut(A,B)}{Vol(A)}+\frac{Cut(A,B)}{Vol(B)},
\end{align}
where $Cut(A,B):=\sum_{i\in A, j\in B}W(i,j)$ and $Vol(A):=\sum_{i\in A, j\in V}W(i,j), Vol(B):=\sum_{i\in B, j\in V}W(i,j)$. But \ref{eq:2} is an NP-hard problem in general. To ease this problem, we tend to solve another equivalent minimization problem
\begin{align} \label{eq:3}
    \min J_{y : y_i\in\{\frac{1}{Vol(A)}, \frac{-1}{Vol(B)}\} }:=\frac{y^tLy}{y^tDy} \hspace{0.5ex} \text{subject to} \hspace{0.5ex} y^tD\bf{1}=0. 
\end{align}
Here we choose $y\in \rr^N$ such that $y_i:=\begin{cases}
    \frac{1}{Vol(A)} , \text{if} \hspace{2ex} i\in A\\
    \frac{-1}{Vol(B)} , \text{if} \hspace{2ex} i\in B. 
\end{cases}$\\

In order to solve \ref{eq:3} we construct the Normalized Graph Laplacians as 
\begin{align*}
    L^\prime:=D^{-1/2}LD^{-1/2}=I-D^{-1/2}WD^{-1/2}
\end{align*} \begin{center}
    or
\end{center}
\begin{align*}
    L^{\prime \prime}:=D^{-1}L=I-D^{-1}W
\end{align*}
and then converting \ref{eq:3} to 
\begin{align}\label{eq:4}
    \min_{z: z^tD^{1/2}\bf{1}=0} J:=\frac{z^t\Tilde{L}z}{z^tz},
\end{align}
where $\Tilde{L}$ is either $L^\prime$ or $L^{\prime \prime}$. Then solving the generalized eigenvalue problem of the form $Lf=\lambda Df$ of \ref{eq:3} is transformed into an eigenvalue problem of the form $\Tilde{L}f=\lambda f$. Since the smallest eigenvalue of $\Tilde{L}$ is $0$, the Rayleigh Quotient \ref{eq:4} is minimized at the eigenvector corresponding to the next smallest eigenvalue. More generally, if the number of clusters is $k$, we have to choose the eigenvectors corresponding to the $k$ smallest eigenvalues. 

Having all these preliminaries and the equivalent formulations, we are now ready to discuss our proposed algorithm. At first, we need to identify carefully the appropriate model space on which we will consider our data points. Since the hyperbolic space of dimension $n$ is not bounded, if we consider data points spreading across the entire space, we might run into problems to verify its consistency. To remove this shortcoming, we will consider only a compact subset of our model space chosen carefully. 

Since the boundary of the model hyperbolic space $\hp^n$ is at infinite distance from the origin (for example, the plane $x_n=0$ in Poincar\'{e} half space model or the unit circle in Poincar\'{e} disc model), we have to consider only a compact subset of $\hp^n$ in order to talk about general notion of convergence. More specifically, we will fix some $\delta>0$ and consider the following subsets of $\hp^n$ for either Poincar\'{e} half space model or disc model:
\begin{itemize}
    \item \textbf{Poincar\'{e} Half Space Model:} Consider $H^\prime:=\{(x_1,x_2,...,x_n)|x_n\geq\delta\}$.
    \item \textbf{Poincar\'{e} Disc Model:} Consider $H^":=\{x\in \dd^n \hspace{2ex}\text{such that} \hspace{2ex}\|x\|\leq (1-\delta)\}$.
\end{itemize}
Note that if we choose $\delta>0$ sufficiently small, then the clusters will not be affected by the choice of $\delta$. Since the formation of clusters depends only on the spatial position of the data points among themselves, choosing a bigger $\delta$ only affect in terms of scaling. However, in order to utilize the hyperbolic nature of $\hp^n$ we will choose $\delta$ to be very small, let's say in the order of $10^{-4}$. 

Since every $n-$ dimensional model hyperbolic space is isometric to each other by the \textbf{Killing-Hopf Theorem}, for the purpose of convergence, it is reasonable to talk about a compact subset of only one of the model hyperbolic spaces. For our purpose, we will only consider the above mentioned compact subset of the Poincar\'{e} Disc Model and denote it by $H$, more explicitly, $H:\{x\in\dd^n \hspace{1ex} \textit{with}\hspace{1ex} \|x\|\leq(1-\delta) \}$ equipped with the Poincar\'{e} metric
\begin{align*}
    d(x_1,x_2):=cosh^{-1}\left(1+\frac{2\|x_2-x_1\|^2}{(1-\|x_1\|^2)(1-\|x_2\|^2)}\right).
\end{align*}

\begin{figure}
     \centering
     %\begin{subfigure}[b]{0.5\textwidth}
         \centering
         \includegraphics[width=0.5\textwidth]{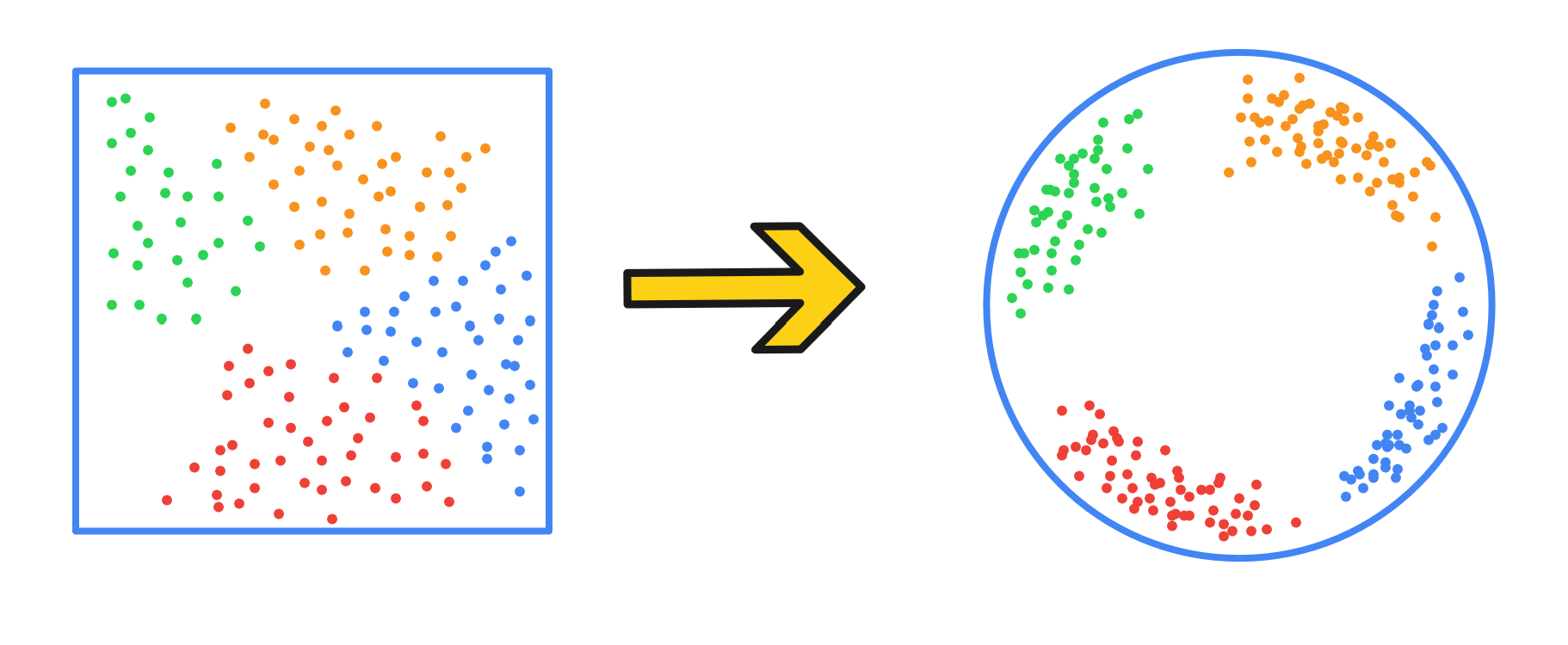}
         \caption{Embedding of the dataset from the Euclidean Space into the Poincar\'{e} Disc, the left figure describes how a natural hierarchy looks like in the Euclidean Space, the right figure describes how the embedded dataset looks like on the Poincar\'{e} Disc. }
         \label{fig: embedding}
     %\end{subfigure}

\end{figure}     
\subsection{The Hyperbolic Spectral Clustering Algorithm (HSCA)}
Let $X = \{x_1, x_2, ..., x_N\}\in \rr^l$ for some $l\in\mathbb{N}$. We will follow the graph-based clustering methods where the edge of a graph (in the Euclidean Space) will be replaced by corresponding Geodesics in $H$. Below we elaborate on the steps of HSCA. For clarity, a pseudo-code of the algorithm is provided in Algorithm 1.  %that are quite similar to the one suggested by Ng, Jordan, and Weiss \cite{ng} at the end. 

\begin{enumerate}
    \item  Embed the original dataset into the Poincar\'{e} Disc of unit radius. Obtain $\mathcal{X^\prime}:=\{x_1^\prime, x_2^\prime,...,x_N^\prime\}\in H$ such that $x_i^\prime:=\frac{x_i}{\|x_i\|+\delta}$, where $\delta>0$ can be taken any small number, let's say $\delta=10^{-2}$, like Figure \ref{fig: embedding}.
    \item  Construct the graph $G(V,E)$ where $V = \{x_1^\prime,x_2^\prime,...,x_N^\prime\}$. The edge set $E:=\{e_{i,j}(x_i^\prime,x_j^\prime)| e_{ij}\}$ where $e_{ij}$ is the geodesic from $x_i^\prime$ to $x_j^\prime$ on $H$.
    \item For the each geodesic-edge $e_{ij}$ (from now on we will use edge instead of geodesic-edge) attach a weight $w_{ij}$ which measures the similarity between the respective nodes $x_i^\prime$ and $x_j^\prime$. The set of weights define the proximity matrix or affinity matrix, this is the $N\times N$ matrix $W$ such that
    \begin{align*}
        W\equiv [W(i,j)] = [w_{ij}], i,j\in\{1,2,...,N\}.
    \end{align*}
     The proximity matrix $W$ is symmetric. A common choice of the proximity matrix is the Gaussian Kernel or the Poisson Kernel(the same Gaussian kernel or Poisson Kernel on Euclidean Space with the metric replaced by geodesic distance on $H$). $W$ can be represented as
   \begin{align*}
       W(i,j) = \begin{cases}
        \exp(-\frac{d_g(x_i^\prime,x_j^\prime)^2}{\sigma^2}), \hspace{2mm} \text{if $d_g(x_i^\prime,x_j^\prime)\leq\epsilon$},\\
        0, \hspace{2mm} \text{otherwise.}
    \end{cases}
   \end{align*}
  for the Gaussian Kernel, or  
    \begin{align*}
        W(i,j) = \begin{cases}
        \exp(-\frac{d_g(x_i^\prime,x_j^\prime)}{2\sigma}), \hspace{2mm} \text{if $d_g(x_i^\prime,x_j^\prime)\leq\epsilon$}\\
        0, \hspace{2mm} \text{otherwise}
        \end{cases}
    \end{align*}
       
    for the Poisson Kernel, 
    where $\epsilon >0$ is pre-defined cut-off distance set by the user. 
    \item Now treat $W$ as the new data matrix, i.e. treat each row (or column since $W$ is symmetric) of $W$ as new data lying in $\rr^N$, call each row of $W$ by $w_i, i\in \{1,2,...,N\}$, where each $w_i$ is a representative of the original data $x_i, i\in\{1,2,...,N\}$ . Compute the modified affinity matrix $W^\prime:=exp\left(-\frac{\|w_i-w_j\|^2}{\sigma^2}\right)$. 
    \item Next compute the degree matrix $D:=diag(d_1,d_2, ...,d_N)$, where $d_i:=\sum_{j=1}^{N}W^\prime(i,j)$. Construct the graph Laplacian $L:=D-W^\prime$. Next construct the normalized graph Laplacian $\Tilde{L}:=D^{-1/2}LD^{-1/2}=I-D^{-1/2}W^\prime D^{-1/2}$. 
    \item Compute the first $k$ eigenvectors $u_1,u_2, ...,u_k$ of the normalized graph Laplacian $\Tilde{L}$ corresponding to the first $k$ eigenvalues of $0=\lambda_1\leq \lambda_2\leq ...\leq \lambda_k $ of $\Tilde{L}$. Call this matrix $U:=[u_1...u_k]\in \rr^{N\times k}$. Normalize the rows of $U$, construct $T\in \rr^{N\times k} $ such that  
    \begin{align*}
        T(i,j):=\frac{U(i,j)}{\sqrt{\sum_{l=1}^k U(i,l)^2}}.
    \end{align*}
    \item For $i=1,2,...,N$, let $y_i\in \rr^k$ be the $i-$th row of $T$. 
    \item Form clusters $C_1, C_2,...C_k$ from the points $\{y_i\}_{i=1}^N$ using $k-$means algorithm. The clusters are $C_j:=\{i|y_i\in C_j\}$. 
\end{enumerate}
\subsubsection{Explanation:} %Here we will talk about the motivation of developing the HSCA. 
Step 1 is clear, since our model space is a subset of the Poincar\'{e} Disc, we have to embed each data point into the space. While doing that we have taken each point and divided it by it's Euclidean norm plus some small positive integer. This will ensure that the data points in the similar hierarchy will be closed in the embedding. Next we have computed the affinity matrix $W$ by treating each edge between points as geodesic edge and have the Poincar\'{e} metric absorbed into the corresponding Kernel. Next we have constructed the modified similarity matrix $W^\prime$ mainly by keeping the idea that if the points $x_i^\prime$ and $x_j^\prime$ in the embedded dataset are closed to a set of same points, then they should ideally belong to the same cluster. The spectral decomposition modified similarity matrix $W^\prime$ will ensure that the representative points in the rows of $W$ will form the cluster closed to the ground level. Rest of the steps are taken from the well known normalized Euclidean Spectral Clustering Algorithm proposed by Ng, Jordan and Weiss \cite{ng}. 

\begin{algorithm*}[!ht]

  \caption{\footnotesize Hyperbolic Spectral Clustering Algorithm (HSCA)}
  \label{alg:surge}
  \scriptsize
  \begin{flushleft}
      {\bfseries Input:} Dataset $\mathcal{X}:=\{x_1, x_2,...,x_N\}\in\rr^n$,  number of clusters=$k$, hyperparameter $\sigma$, cut-off length=$\epsilon$. \\
  \end{flushleft}
  \begin{flushleft}
      {\bfseries Output:} Cluster labels $\mathcal{C}:=\{C_1,C_2,...,C_k\}$ where $C_i:=\{j|x_j\in C_i\}$.\\
  \end{flushleft}
  
  \let \oldnoalign \noalign
  \let \noalign \relax
  
  \let \noalign \oldnoalign
 
  \begin{algorithmic}[1]
 
  \STATE \textbf{Transformation:} Obtain $\mathcal{X^\prime}:=\{x_1^\prime, x_2^\prime,...,x_N^\prime\}\in H$ such that $x_i^\prime:=\frac{x_i}{\|x_i\|+\delta}$, where $\delta>0$ small, let's say $\delta=10^{-2}$.
  \STATE \textbf{Constructing The Similarity Matrix:} Construct $W\in\rr^{N\times N}$ with $W(i,j):=\begin{cases}
      exp\left(-\frac{d_g(x_i^\prime, x_j^\prime)^2}{\sigma^2}\right), \hspace{2ex} \text{if} \hspace{1ex} d_g(x_i^\prime, x_j^\prime)\leq\epsilon\\
      0, \hspace{2ex} \text{otherwise}.
  \end{cases}$ or \hspace{2ex}
    $W(i,j) := \begin{cases}
        \exp(-\frac{d_g(x_i,x_j)}{2\sigma}), \hspace{2mm} \text{if $d_g(x_i,x_j)\leq\epsilon$}\\
        0, \hspace{2mm} \text{otherwise}
    \end{cases}$    
    \\
    \STATE \textbf{Constructing New Dataset and Computing Modified Affinity Matrix:} Treat each row $w_i$ of $W$ as new data point and compute $W^\prime:=exp\left(-\frac{\|w_i-w_j\|^2}{\sigma^2}\right)$. 
  \STATE \textbf{Constructing the Degree Matrix:} Construct the diagonal degree matrix $D\in\rr^{N\times N}$ with $D(i,j):=\begin{cases}
      \sum_{j=1}^N W^\prime(i,j), \hspace{2ex} \text{if} \hspace{1ex} i=j\\
      0, \hspace{2ex} \text{otherwise}.
  \end{cases}$\\
  \STATE \textbf{Construct the Normalized Graph Laplacian:} Obtain $L:=D-W^\prime\in\rr^{N\times N}$. Then Construct $\Tilde{L}:=D^{-1/2}LD^{-1/2}\in\rr^{N\times N}$. 
  \STATE \textbf{Spectral Decomposition of the Normalized Graph Laplacian:} Obtain the first $k$ eigenvalues of $\Tilde{L}$, $0=\lambda_1\leq\lambda_2\leq...\leq \lambda_k$ and the corresponding eigenvectors $u_k\in\rr^N$. Let $U:=[u_1,u_2,...u_k]\in\rr^{N\times k}$.\\
  \STATE \textbf{Normalizing the Eigen Matrix:} Normalize the rows of $U$, obtain $T\in\rr^{N\times k}$ such that $T(i,j):=\frac{U(i,j)}{\sqrt{\sum_{l=1}^k U(i,l)^2}}$. 
  \STATE \textbf{Representative Points on $\rr^k$:} Let $\mathcal{Y}:=\{y_1,y_2,...y_N\}\in\rr^k$, where $y_i$ represents $x_i$ for $i\in\{1,2,...,N\}$ and $y_i^j=U(i,j)$. \\
  \STATE \textbf{Cluster Formation:} Obtain the clusters $C_1,C_2,...,C_k$ by performing $k-$ means clustering on $\mathcal{Y}$, where $C_i:=\{j|y_j\in C_i\}$. 

\end{algorithmic}

\end{algorithm*}

\subsubsection*{Time Complexity:} The most computationally intense step is to perform the spectral decomposition of the similarity matrix which takes time $\mathcal{O}(n^3)$ \cite{siron}. Calculating the pairwise Poincare Distance and the similarity matrix each takes $\mathcal{O}(n^2)$ time. Finally applying the $k-$means during the eigenvalue decomposition takes $\mathcal{O}(nldk)$ amount of time, where $n$ is the number of data points, $k$ is the number of clusters, $l$ is the number of $k-$means iterates and $d$ is the dimension of the Hyperbolic Space $H$. 
\vspace{-3mm}
\subsection{Approximated Hyperbolic Spectral Clustering with Poincar\'{e} $k$-Means Based landmark Selection (HLS K)}

Here we shortly describe another variant of the HSCA proposed with an analogy to the Approximated Spectral Clustering with $k$-Means based landmark Selection, one variant of the ESCA. We assume the form of the input dataset as $\mathcal{X}:=\{x_1, x_2,...,x_N\}\in\rr^n$. We also assume the number of landmark points are $m \in \mathbb{N}$, where $m$ has to be chosen carefully such that $m>>k$, where $k$ is the number of original clusters and also $m$ should not be very small compared to $N$. [For example, if the number of samples is $5000$, and $k=10$, we can choose $m=100$ or $200$.]  The steps of the algorithm are provided as follows: [For details we refer to the supplementary]

\begin{enumerate}
    \item We obtain $\mathcal{X^\prime}:=\{x_1^\prime, x_2^\prime,...,x_N^\prime\}\in H$ such that $x_i^\prime:=\frac{x_i}{\|x_i\|+\delta}$, where $\delta>0$ small, let's say $\delta=10^{-2}$.
    \item Next we perform the initial Poincar\'{e} $k-$Means [This is a $k-$Means clustering with the Euclidean Distance replaced by the Poincar'{e} Distance] to form $m$ clusters, namely $I_1, I_2,..., I_m$ with centroids of these clusters as $y_1,y_2,...,y_m$ respectively.
    \item Now we construct $V\in \rr^{m\times N}$ such that  $V(i,j):=\begin{cases}
      exp\left(-\frac{d_g(y_i, x_j^\prime)^2}{\sigma^2}\right), \hspace{2ex} \text{if} \hspace{1ex} d_g(y_i, x_j^\prime)\leq\epsilon\\
      0, \hspace{2ex} \text{otherwise}.
  \end{cases}$ or \hspace{2ex}
    $W(i,j) := \begin{cases}
        \exp(-\frac{d_g(y_i,x_j^\prime)}{2\sigma}), \hspace{2mm} \text{if $d_g(y_i,x_j^\prime)\leq\epsilon$}\\
        0, \hspace{2mm} \text{otherwise}
    \end{cases}$.
    \item Then we normalize the columns of $V$, i.e., form the matrix $E\in \rr^{m\times N}$ with $E(i,j):=\frac{V(i,j)}{\sum_{i=1}^mV(i,j)}$. Construct the diagonal matrix $D_E=diag(d_1,d_2,...,d_m)\in\rr^{m\times m}$ with $d_i=:(\sum_{j=1}^mE(i,j))^{-1/2}$. Construct $Z=DE\in\rr^{m\times N}$. Form the final matrix $F:=Z^tZ$.
    \item Finally we follow steps $3-9$ of HSCA (Algorithm \ref{alg:surge}) to form the clusters.
    
\end{enumerate}

\section{Verifying The Consistency of HSCA}  \label{sec:5}
Following our previous notations, $x$ and $y$ are two points on the Poincar\'{e} Disc and the metric is given as
\begin{align*}
    d(x,y)=2\sinh^{-1}\left(\sqrt{\frac{\delta(x,y)}{2}}\right),
\end{align*}
where
\begin{align*}
    \delta(x,y)=2\frac{\|x-y\|^2}{(1-\|x\|^2)(1-\|y\|^2)}.
\end{align*}

The Hyperbolic Gaussian Kernel $K_{H_G}$ is given as
\begin{align*}
    K_{H_G}(x,y)=\exp(-ad(x,y)^2), a>0
\end{align*}
and the Hyperbolic Poisson Kernel $K_{H_P}$ is given as
\begin{align*}
    K_{H_P}(x,y)=\exp(-ad(x,y)), a>0.
\end{align*}
At first, we will prove the consistency of the HSCA using Gaussian Kernel. The proof will be quite similar if we use the Poisson Kernel.  Before that we will look at a couple of results involved in the consistency proof.  
\begin{lemma}\label{lem:5.1}
    For the usual Euclidean Gaussian Kernel given by $K(x,y)=exp(-a\|x-y\|^2)$, we have $K_{H_G}(x,y)\leq K(x,y)$ whenever $x,y\in H$.
\end{lemma}
\begin{proof}
    See Appendix A.
\end{proof}

\begin{rem}
    Lemma \ref{lem:5.1} also holds true for the Poisson Kernel. In step 3 we only need to use the function$f(x):=\sinh^{-1}(x)-\frac{x^2}{4}$ [which is true, since $f^\prime(x)=\frac{1}{\sqrt{1+x^2}}-\frac{x}{2}\geq \frac{2-\sqrt{2}}{2\sqrt{2}}$ for $0\leq x\leq 1$ ] instead of $f(x):=\sinh^{-1}(x)-\frac{x}{2}$. There will be an extra constant $1/2$ in the exponent, but this will not affect the proof of Lemma \ref{lem:5.4} presented later. 
\end{rem}

\begin{rem}\label{rem:5.2}
    $K_{H_G}$ is radial: If we fix one variable, let's say $y$ at $0$, then $\delta(x,0)=2\frac{\|x\|^2}{1-\|x\|^2}$, which is a radial function. Therefore, the Poincar\'{e} Metric is also radial, so is the Hyperbolic Gaussian Kernel.
\end{rem}

\begin{lemma}\label{lem:5.2}
    The hyperbolic Gaussian Kernel $K_{H_G}\in L^1(H)$, i.e. this kernel is an absolutely integrable.
\end{lemma}
\begin{proof}
    See Appendix A. 
\end{proof}

\subsubsection*{Terminology:} For a compact subset $\Omega\in\rr^n$ [with $0$ in its interior], we call $\Omega$ to be symmetric if for every $x\in\Omega$ and for every $M\in SO_n(\rr^n)$, we have $Mx\in\Omega$.

\begin{lemma}\label{lem:5.3}
  Suppose $\Omega\in\rr^n$ is symmetric, $f\in L^1(\Omega)$ and $f$ is radial. Then its Fourier Transform is also radial.
\end{lemma}
\begin{proof}
    See Appendix A.
\end{proof}

Next we intend to use Theorem 3 \cite{zhou} and this necessitates computing the Fourier Transform $\widehat{k}(w)$ of $k_{H_G}(x)$ and will show that $\widehat{k}$ decays exponentially. 
\begin{lemma}\label{lem:5.4}
    There exist $C,l>0$ such that $\widehat{k}(w)\leq C\exp(-l|w|)$ for all $w\in\rr^n$.
\end{lemma}
\begin{proof}
    See Appendix A.
\end{proof}

\subsubsection*{Terminology and Definitions:} Let $H$ be the compact subset of the Poincar\'{e} Disc as defined above. $k:H\times H\to \rr$ be the similarity function with the Gaussian Kernel equipped with the Poincar\'{e} Metric. Let $h:H\times H\to \rr$ be the normalized similarity function. Then for any continuous function $g\in \mathcal{C}(H)$, we define the following [as in section 6 \cite{lux}:
\begin{align*}
    \mathcal{K}&:=\{k(x,\cdot):x\in H\},\\
     \mathcal{H}&:=\{h(x,\cdot):x\in H\},\\
    g\cdot \mathcal{H}&:=\{g(x)h(x,\cdot):x\in H\},\\
     \text{and}\mathcal{H}\cdot\mathcal{H}&:=\{h(x,\cdot)h(x,\cdot), x\in H\}.
\end{align*}
We also define $\mathcal{F}:= \mathcal{K}\cup(g\cdot \mathcal{H})\cup (\mathcal{H}\cdot \mathcal{H})$. 
Now we will re-write Theorem 19\cite{lux} with a slightly modified proof. 
\begin{thm}\label{them:5.1}
    Let $(H,\mathcal{A}, P)$ be a probability space with $\mathcal{A}$ being any arbitrary sigma algebra on $H$. Let $\mathcal{F}$ be defined as above with $\|f\|_{\infty}\leq 1$. Let $X_n$ be a sequence of i.i.d. random variables drawn according to the distribution $P$ and $P_n$ be the corresponding emperical diustributions. Then there exists a constant $c>0$ such that for all $n\in \mathbb{N}$ with probability at least $\delta$, 
    \begin{align*}
        \sup_{f\in \mathcal{F}}|P_nf-Pf|&\leq \frac{c}{\sqrt{n}}\int_{0}^{\infty}\sqrt{\log(\mathcal{N}, \epsilon, L^2(P_n))}d\epsilon\\
        &+\sqrt{\frac{1}{2n}\log\left(\frac{2}{\delta}\right)},
    \end{align*}
    where $\mathcal{N}$ is the covering number of the space $H$ with ball of radius $\epsilon$ with respect to the metric $L^2(P_n)$. Hence \textbf{the rate of convergence of the Hyperbolic Spectral Clustering is $\mathcal{O}\left(\frac{1}{\sqrt{n}}\right)$}. 
\end{thm}
\begin{proof}
    See Appendix A.
\end{proof}

%Finally, Theorem 16 of \cite{lux} combining with Theorem \ref{them:5.1} gives us the following result, completing the proof of the convergence of HSCA.\\ 

\begin{rem}
    Note that in deriving the convergence rate of the hyperbolic spectral clustering, we used results used mostly in the proof of the consistency of spectral clustering in the Euclidean set-up. The hyperbolic metric is in general very powerful compared to the squared euclidean metric, which forces the hyperbolic Gaussian/Poisson Kernel converging to $0$ much faster than the Euclidean ones. Therefore, we believe the convergence rate of the hyperbolic spectral clustering can be improved, which requires estimating a careful bound on the logarithmic covering number with respect to the hyperbolic metric.  
 \end{rem}

\section{Experiments and Results} \label{sec:6}
\subsection{Description of the Datasets:} %In order to examine the clustering efficiency of our proposed algorithm, we will use a combination of total $7$ datasets, among these $4$ are real datasets \textit{Airport, Glass, Zoo} and \textit{Wisconsin} and $3$ are synthetic datasets \textit{2d-20c-no0, D31}, and \textit{st900}. The ground cluster levels are not available in case of the Airport dataset. Details of the datasets are shown in Table \ref{tab:1}. Since the Airport dataset is too large, we took random samples of sizes $1000$, $2000$, $5000$ and $10000$ four times to evaluate the clustering efficiency assuming the number of ground clusters being $5$.   
To examine the clustering efficiency of our proposed algorithm, we will use a combination of total \( 7 \) datasets, among these \( 4 \) are real datasets \textit{Airport, Glass, Zoo} and \textit{Wisconsin} and \( 3 \) are synthetic datasets \textit{2d-20c-no0, D31}, and \textit{st900}. The ground cluster levels are not available in case of the Airport dataset. Details of the datasets are shown in Table \ref{tab:1}. Since the Airport dataset is too large, we took random samples of sizes \( 1000 \), \( 2000 \), \( 5000 \) and \( 10000 \) four times to evaluate the clustering efficiency assuming the number of ground clusters being \( 5 \).
\begin{table}[t]
     \caption{Intrinsic Evaluation Metrics for Datasets}
    \centering
    \scriptsize
    \resizebox{\columnwidth}{!}{%
    
    \begin{tabular}{l ccc}
    \toprule
    \multirow{2}{*}{\textbf{Datasets}} & \multicolumn{1}{c}{\textbf{No. of Samples}} & \multicolumn{1}{c}{\textbf{Dimension}} & \multicolumn{1}{c}{\textbf{Number of Clusters}}\\
    \cmidrule{2-4}
    Airport  & 1048575 & 196 & -\\
    Wisconsin & 699 & 9 & 2\\
    Glass & 214 & 9 & 6\\
    Zoo & 101 & 16 & 7\\
    2d-20c-no0 & 1517 & 2 & 20\\
    st900 & 900 & 2 & 9\\
    D31 & 3100 & 2 & 31\\
    \bottomrule

    \end{tabular}%
    } \label{tab:1}
    \end{table}

\begin{table}[t]
     \caption{Intrinsic Evaluation Metrics for Datasets}
    \centering
    \scriptsize
    \resizebox{\columnwidth}{!}{%
    
    \begin{tabular}{l ccc|ccc}
        \toprule
        % Airport(1000) & Airport(2000) %%%%%%%%%%%%%%%%%%%%%%%%%%%%%
         \multirow{2}{*}{\textbf{Methods}} & \multicolumn{3}{c}{\textbf{Airport(1000)}} & \multicolumn{3}{c}{\textbf{Airport(2000)}}\\
         \cmidrule{2-7}
         % \cmidrule{}
          & \textbf{S. Score} & \textbf{D.B. Score} & \textbf{C.H. Index} & \textbf{S. Score} & \textbf{D.B. Score} & \textbf{C.H. Index} \\
         \midrule
         ESCA(G) &  0.68 & 0.75 &  78.62 & \bf 0.83 &   \bf 0.12 & 189.66 \\
         HSCA(G) &  0.24 &  \bf0.31 & \bf1114.77 & 0.21 &  0.32 &  2023.64\\
         \midrule
         Improvement & 64.7\% ($\dn$) & 44\%($\upa$) & 1317.9\%($\upa$)& 74.7\%($\dn$)& 166.7\%($\dn$)& 967.0\%($\upa$)\\
         \midrule
         ESCA(P) & \bf 0.70 &  0.40 &  140.02 &  0.82 &  0.40 & 93.31\\
         HSCA(P) & 0.21 & 0.32 & 1111.89 & 0.17 &0.34 & \bf2024.71\\
         \midrule
          Improvement & 70.0\% ($\dn$) & 20\%($\upa$) & 694.1\%($\upa$)& 79.27\%($\dn$)& 17.7\%($\upa$)& 2069.87\%($\upa$)\\
         %Improvement &\color{Green}{0.59} & \color{Green}{25.81} & \color{Green}{1111.42} & \color{Green}{0.79} & \color{Green}{3.87} & \color{Green}{2023.53}\\
         %\midrule
         %Average Gain &\color{Green}{0.61} &\color{Green}{25.87} &\color{Green}{1112.77} &\color{Green}{0.81} &\color{Green}{7.78} &\color{Green}{2023} \\
         \toprule

         %Airport(5000) & Airport(10000) %%%%%%%%%%%%%%%%%%%%%%

         \multirow{2}{*}{\textbf{Methods}} & \multicolumn{3}{c}{\textbf{Airport(5000)}} & \multicolumn{3}{c}{\textbf{Airport(10000)}}\\
         \cmidrule{2-7}
         % \cmidrule{}
          & \textbf{S. Score} & \textbf{D.B. Score} & \textbf{C.H. Index} & \textbf{S. Score} & \textbf{D.B. Score} & \textbf{C.H. Index} \\
         \midrule
         ESCA(G) & \bf 0.57 &  0.56 &  144.84 &   0.83 &  0.46 & 218.53 \\
         HSCA(G) &  0.24 & \bf 0.32 & \bf 5299.59 &  0.24 & 0.33  & \bf 10584.5\\
         \midrule
         %\midrule
          Improvement & 57.9\% ($\dn$) & 42.9\%($\upa$) & 3558.9\%($\upa$)& 71.1\%($\dn$)& 28.3\%($\upa$)& 4755.1\%($\upa$)\\
         %Improvement &\color{Green}{0.60} &\color{Green}{38.40} &\color{Green}{5298.49} &\color{Green}{0.76} &\color{Green}{7.81} &\color{Green}{10583.97}\\
         \midrule
         ESCA(P) &  0.57 & \bf 0.56 & 144.83 & \bf 0.83 &  \bf 0.26 & 218.53 \\
         HSCA(P) &  0.21 & 0.34 &  5258.51 &  0.21 & 0.35 &  10428.3\\
         \midrule
          Improvement & 63.2\% ($\dn$) & 39.3\%($\upa$) & 3531.1\%($\upa$)& 74.7\%($\dn$)& 34.6\%($\dn$)& 4672.1\%($\upa$)\\
         %\midrule
          %Improvement &\color{Green}{0.57}&\color{Green}{38.11}&\color{Green}{5257.39}&\color{Green}{0.73}&\color{Green}{7.58}&\color{Green}{10427.77}\\
         %\midrule
         %Average Gain &\color{Green}{0.59}&\color{Green}{40.26}&\color{Green}{5277.94}&\color{Green}{0.75}&\color{Green}{7.7}&\color{Green}{10505.87}\\
         %\bottomrule
         \toprule

         % Wisc Breast-Cancer Dataset, Glass Dataset %%%%%%%%%%%%%%%%%5
         \multirow{2}{*}{\textbf{Methods}} & \multicolumn{3}{c}{\textbf{Wisconsin Breast Cancer}} & \multicolumn{3}{c}{\textbf{Glass}}\\
         \cmidrule{2-7}
         % \cmidrule{}
          & \textbf{S. Score} & \textbf{D.B. Score} & \textbf{C.H. Index} & \textbf{S. Score} & \textbf{D.B. Score} & \textbf{C.H. Index} \\
         \midrule
         ESCA(G) &  0.41 &  0.49 & 3.44 & 0.53 & 0.28 & 9.94\\
         HSCA(G) &  0.01 &  0.91 & 131.54 &  0.42 &  \bf 0.10 & \bf 111.66\\
         \midrule
         Improvement & 97.56\% ($\dn$) & 85.7\%($\dn$) & 3723.8\%($\upa$)& 20.8\%($\dn$)& 64.3\%($\upa$)& 1023.3\%($\upa$)\\
         %Improvement &\color{Green}{0.22} &\color{Green}{2.35}&\color{Green}{130.31}&\color{Green}{0.54}&\color{Green}{11.03}&\color{Green}{109.94}\\
         \midrule
         ESCA(P) & \bf 0.46 & \bf 0.40 & 5.12 & \bf 0.57 & 0.26 & 10.56\\
         HSCA(P) & 0.27 & 0.65 & \bf219.06 & 0.28  & 0.12 & 110.72\\
         \midrule
         Improvement & 41.3\% ($\dn$) & 62.5\%($\dn$) & 4280.1\%($\upa$)& 50.1\%($\dn$)& 53.8\%($\upa$)& 948.5\%($\upa$)\\
          %Improvement &\color{Green}{0.35}&\color{Green}{4.93}&\color{Green}{217.83}&\color{Green}{0.40}&\color{Green}{10.84}&\color{Green}{109.00}\\
         %\midrule
         %Average Gain &\color{Green}{0.29}&\color{Green}{3.64}&\color{Green}{174.07}&\color{Green}{0.47}&\color{Green}{10.94}&\color{Green}{109.47}\\
         %\bottomrule
         \toprule

         % Zoo and 2d-20c-no0 Datasets
         \multirow{2}{*}{\textbf{Methods}} & \multicolumn{3}{c}{\textbf{Zoo }} & \multicolumn{3}{c}{\textbf{2d-20c-no0}}\\
         \cmidrule{2-7}
         % \cmidrule{}
          & \textbf{S. Score} & \textbf{D.B. Score} & \textbf{C.H. Index} & \textbf{S. Score} & \textbf{D.B. Score} & \textbf{C.H. Index} \\
         \midrule
         ESCA(G) &0.22  & 0.64 & 13.51 & 0.54 &  0.64 & 1953.87 \\
         HSCA(G) & 0.32 &  \bf 0.16 & 49.62 & 0.47 & 0.16 &\bf 9630.12\\
         \midrule
         Improvement & 45.5\% ($\upa$) & 75.0\%($\upa$) & 267.3\%($\upa$)& 12.9\%($\dn$)& 75.0\%($\upa$)& 393.9\%($\upa$)\\
         %Improvement &\color{Green}{0.53} &\color{Green}{1.29}&\color{Green}{47.56}&\color{Green}{1.09}&\color{Green}{23.46}&\color{Green}{9629.26}\\
         \midrule
         ESCA(P) &0.23  & 0.96 & 5.94 &  0.35 &  0.71 & 2.27\\
         HSCA(P) &\bf 0.35 & 0.17&\bf 50.85 &\bf 0.54 & \bf 0.10 & 8643.90\\
         \midrule
         Improvement & 52.2\% ($\upa$) & 82.3\%($\upa$) & 756.1\%($\upa$)& 54.3\%($\upa$)& 85.9\%($\upa$)& $38.1\times 10^5$\%($\upa$)\\
          %Improvement &\color{Green}{0.56}&\color{Green}{1.43}&\color{Green}{48.79}&\color{Green}{1.16}&\color{Green}{24.07}&\color{Green}{8643.04}\\
         %\midrule
         %Average Gain &\color{Green}{0.55}&\color{Green}{1.36}&\color{Green}{48.18}&\color{Green}{1.13}&\color{Green}{23.77}&\color{Green}{9136.15}\\
         %\bottomrule

         \toprule
             % st900 & D31 Datasets
         \multirow{2}{*}{\textbf{Methods}} & \multicolumn{3}{c}{\textbf{st900 }} & \multicolumn{3}{c}{\textbf{D31}} \\
         \cmidrule{2-7}
         % \cmidrule{}
           & \textbf{S. Score} & \textbf{D.B. Score} & \textbf{C.H. Index} & \textbf{S. Score} & \textbf{D.B. Score} & \textbf{C.H. Index}\\
         \midrule
         ESCA(G) & -0.03 & 0.68 & 75.31 & 0.39 & \bf0.56 & 2365.22\\
         HSCA(G) &  0.47 &  0.86 & 1548.80 & 0.17 & 0.85  & \bf 2102.17\\
         \midrule
         Improvement & 1667.7\% ($\upa$) & 26.5\%($\dn$) & 1956.6\%($\upa$)& 56.4\%($\dn$)& 51.8\%($\dn$)& 11.12\%($\upa$)\\
         %Improvement &\color{Green}{0.92} &\color{Green}{9.11}&\color{Green}{1547.77}&\color{Green}{1.02}&\color{Green}{63.48}&\color{Green}{2101.37}\\
         \midrule
         ESCA(P) & -0.03 &  0.67 & 75.28 & 0.17 & 0.95 & 185.07\\
         HSCA(P) &\bf 0.49  & \bf 0.63&\bf 2388.67 &\bf 0.50 & 0.57 &\bf 69144.07\\
         \midrule
         Improvement & 1733.3\% ($\upa$) &6.0\%($\upa$) & 3073.1\%($\upa$)& 194.1\%($\upa$)& 40.0\%($\upa$)& 37261\%($\upa$)\\
          %Improvement &\color{Green}{0.94}&\color{Green}{9.33}&\color{Green}{2387.67}&\color{Green}{1.35}&\color{Green}{64.57}&\color{Green}{69143.27} \\
         
         %Average Gain &\color{Green}{0.93}&\color{Green}{9.22}&\color{Green}{1967.72}&\color{Green}{1.19}&\color{Green}{64.03}&\color{Green}{35622.32}\\
         \bottomrule

    \end{tabular}%
    }

    \label{tab:table 0}
\end{table}  

\clearpage 
\subsection{Performance Measure and Validity Index:} When the true cluster labels of the datasets are not known, we will use three intrinsic cluster performance measures to compare the Euclidean Spectral Clustering algorithm with the Hyperbolic Spectral Clustering Algorithm. We will use three main intrinsic evaluation metrics, namely
\begin{itemize}
\item Silhouette Coefficient [S. Score] \cite{nich}
\item Davies-Bouldin score [D.B. Score] \cite{rojas}
\item Calinski–Harabasz index [C.H. Index] \cite{asik}
\end{itemize}

Also, for simulated datasets, we know the true cluster labels and the number of clusters. There we will compare ESCA and HSCA by the extrinsic evaluation metrics, namely
\begin{itemize}
    \item Adjusted Rand Index (ARI) \cite{warren}
    \item Normalized Mutual Information (NMI) \cite{tesmer}
\end{itemize}

\begin{figure*}[]
     \centering         
\begin{subfigure}[b]{0.3\textwidth}
         \centering
         \includegraphics[width=\textwidth]{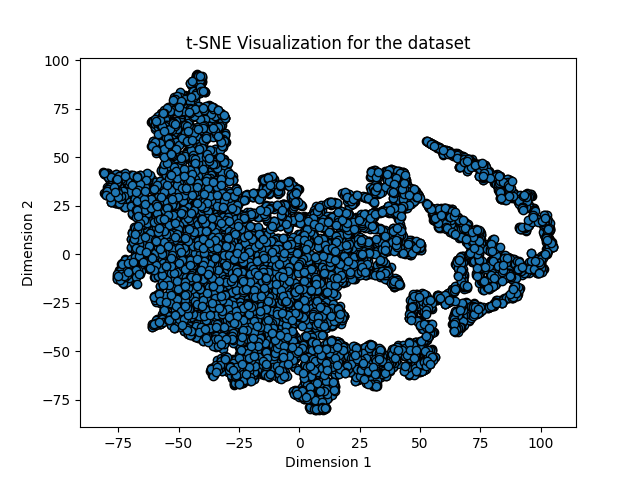}
         \caption{Airport(10000)}
         \label{fig:dataset10000c}
     \end{subfigure}
     \hfill
     \begin{subfigure}[b]{0.3\textwidth}
         \centering
         \includegraphics[width=\textwidth]{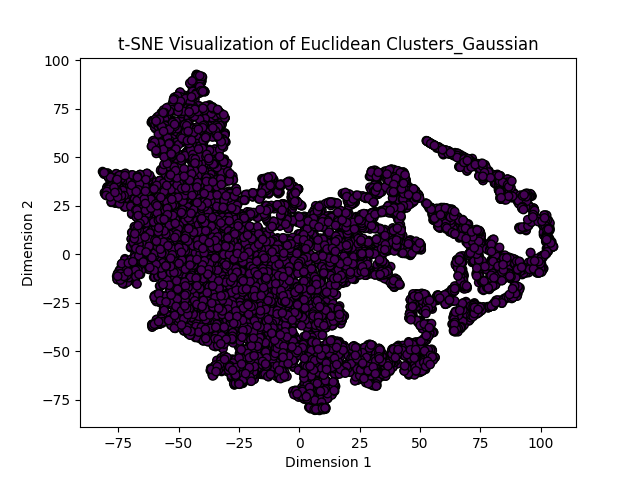}
         \caption{A(10000)-ESCA Clusters}
         \label{fig: euclidean10000c}
     \end{subfigure}
     \hfill
     \begin{subfigure}[b]{0.3\textwidth}
         \centering
         \includegraphics[width=\textwidth]{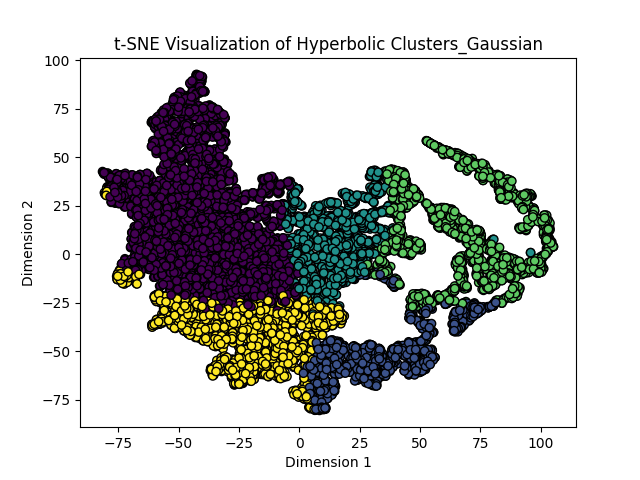}
         \caption{A(10000)-HSCA clusters}
         \label{fig:hyperbolic10000c}
     \end{subfigure}
        \caption{t-SNE Visualization of the Airport Dataset and Clusters}
        \label{fig:4}
        \centering
    \begin{subfigure}[b]{0.3\textwidth}
         \centering
         \includegraphics[width=\textwidth]{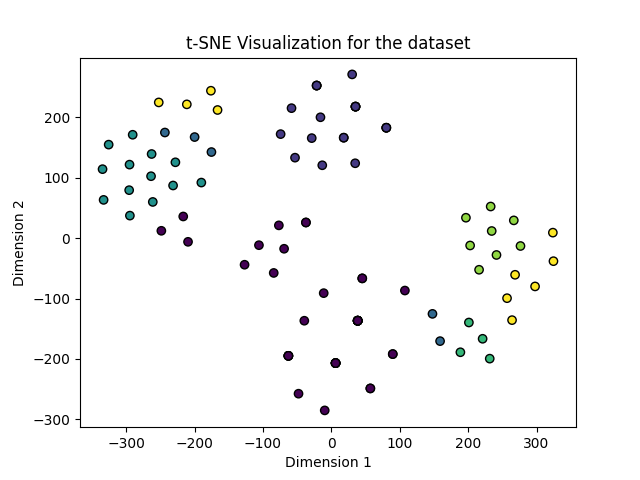}
         \caption{Dataset Zoo}
         \label{fig:dataset_zooc}
     \end{subfigure}
     \hfill
      \begin{subfigure}[b]{0.3\textwidth}
         \centering
         \includegraphics[width=\textwidth]{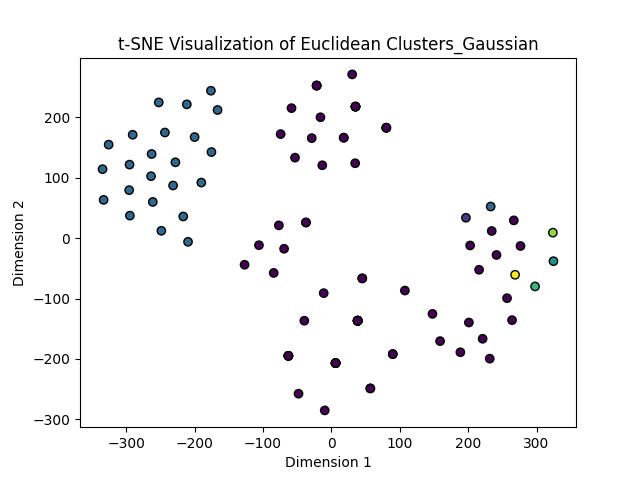}
         \caption{Zoo-ESCA Clusters}
         \label{fig:euclid_zooc}
     \end{subfigure}
     \hfill
      \begin{subfigure}[b]{0.3\textwidth}
         \centering
         \includegraphics[width=\textwidth]{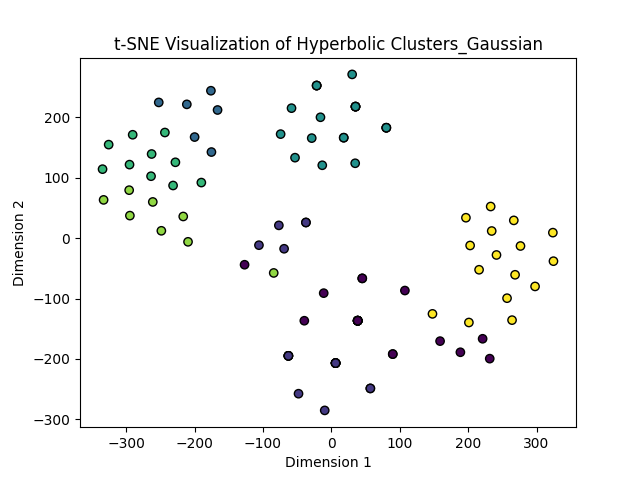}
         \caption{Zoo-HSCA Clusters}
         \label{fig:hyp_zooc}
     \end{subfigure}
     \caption{t-SNE Visualization of the Zoo Dataset and Clusters}
     \label{fig:5}
     \begin{subfigure}[b]{0.3\textwidth}
         \centering
         \includegraphics[width=\textwidth]{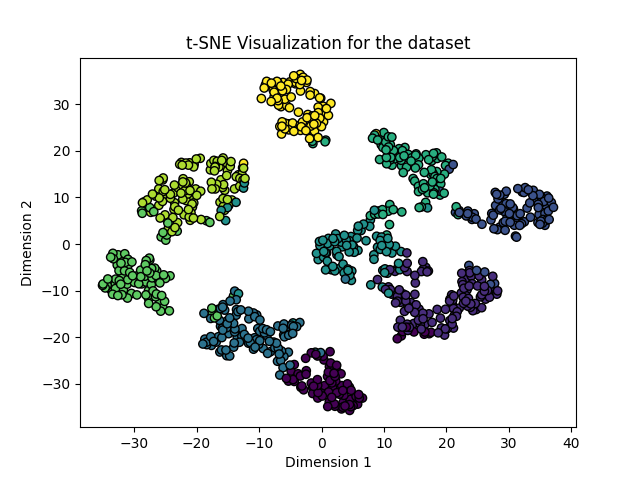}
         \caption{Dataset st900}
         \label{fig:dataset_st900c}
     \end{subfigure}
     \hfill
     \begin{subfigure}[b]{0.3\textwidth}
         \centering
         \includegraphics[width=\textwidth]{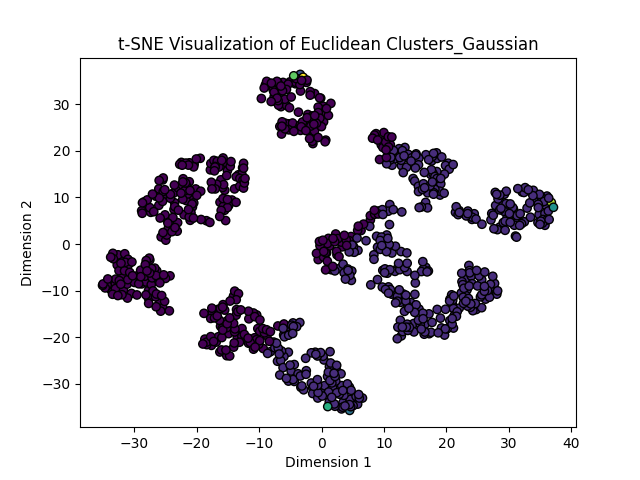}
         \caption{st900-ESCA Clusters}
         \label{fig: euclid_st900c}
     \end{subfigure}
     \hfill
     \begin{subfigure}[b]{0.3\textwidth}
         \centering
         \includegraphics[width=\textwidth]{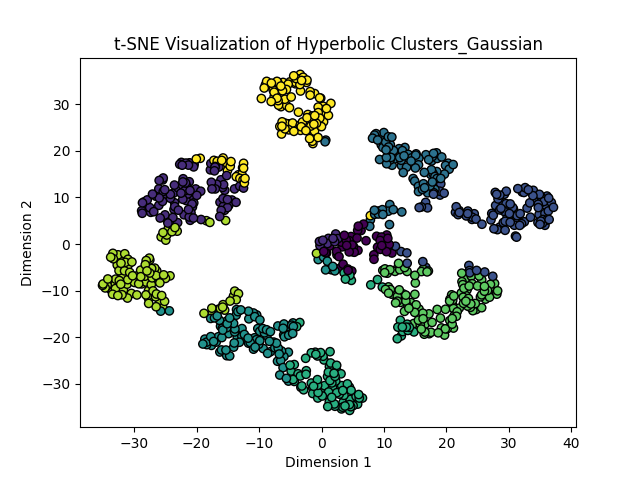}
         \caption{st900-HSCA clusters}
         \label{fig:hyp_st900c}
     \end{subfigure}
     \caption{t-SNE Visualization of the st900 Dataset and Clusters}
        \label{fig:6}
        
\end{figure*}

%%%%%--Silhoutte Score%%%%%%%%%%%%%%%%

\clearpage

\begin{table*}[t]
     \caption{Extrinsic Evaluation Metrics for Datasets for the variants of ESCA and ELSC-K}
    \centering
    \scriptsize
    \resizebox{\columnwidth}{!}{%
    \begin{tabular}{l c c |c  c | c c |cc |cc |cc }
        \toprule
        \multirow{2}{*}{\textbf{Methods}} & \multicolumn{2}{c}{\textbf{Wisconsin }} & \multicolumn{2}{c}{\textbf{Glass}} &\multicolumn{2}{c}{\textbf{Zoo}} &\multicolumn{2}{c}{\textbf{2d-20c-no0}} &\multicolumn{2}{c}{\textbf{st900}} &\multicolumn{2}{c}{\textbf{D31}} \\
         \cmidrule{2-13}
         % \cmidrule{}
           & \textbf{ARI} & \textbf{NMI} & \textbf{ARI} & \textbf{NMI} & \textbf{ARI} & \textbf{NMI}  & \textbf{ARI} & \textbf{NMI} & \textbf{ARI} & \textbf{NMI} & \textbf{ARI} & \textbf{NMI}\\
         \midrule
         ESCA(G) & 0.01 & 0.01 & 0.01 & 0.06 & 0.26 & 0.37 & 0.67 & 0.86 & 0.16 & 0.35 & 0.56 & 0.84\\
         HSCA(G) & 0.77 & 0.66 & 0.23 & 0.36 & 0.53 &  0.70 & \bf0.76 & \bf0.87 & 0.72 & 0.76 & 0.22 & 0.60\\

         \midrule
        Improvement & 7600\% ($\upa$) & 6500\%($\upa$) & 2200\%($\upa$)& 500\%($\upa$)& 104\%($\upa$)& 89\%($\upa$) &13\%($\upa$)&1.1\%($\upa$) & 350\%($\upa$)& 117\%($\upa$) & 60\%($\dn$)& 29\%($\dn$)\\
         \midrule
         
         ESCA(P)& 0.01 & 0.01 & 0.01 & 0.05 & 0.14 & 0.34 & 0.12 & 0.42 & 0.17 & 0.35 & 0.09 & 0.45\\
         HSCA(P) & 0.16 & 0.24 & 0.25 & 0.38 & 0.57 &  0.76 &  0.60 &  0.82 &  0.63 &  0.71 & 0.29 & 0.63\\ 
         \midrule
          Improvement & 1500\%($\upa$) & 2300\%($\upa$)& 2400\%($\upa$)& 660\%($\upa$) & 307\%($\upa$)& 124\%($\upa$)& 400\%($\upa$)& 95\%($\upa$) & 354\%($\upa$)&  103\%($\upa$)& 222\%($\upa$) & 40\%($\upa$)\\
         \midrule
         %Average Gain &\color{Green}{0.47} &\color{Green}{0.45}&\color{Green}{0.24}&\color{Green}{0.31}&\color{Green}{0.35}&\color{Green}{0.37}&\color{Green}{0.29}&\color{Green}{0.20}&\color{Green}{0.51}&\color{Green}{0.39}&\color{Red}{-0.07}&\color{Red}{-0.06}\\
        
         \toprule
          ELSC-K(G) & 0.82 & 0.72 & 0.23 & 0.40 & 0.70 & \bf0.78  & 0.54 & 0.79 & 0.72 & 0.78 & 0.94 & 0.96\\
         HLSC-K(G) & 0.84 & 0.73 & 0.25 & 0.39 & 0.71 & \bf 0.78 & 0.55 &  0.80 & \bf0.75 & \bf0.81 & 0.95 & 0.97\\
         \midrule
         Improvement & 2.4\%($\upa$) &1.4\%($\upa$) &8.7\%($\upa$) &2.5\%($\dn$) & 1.4\%($\dn$)& 0.00\%(-)&1.9\%($\upa$) & 1.3\%($\upa$)& 4.2\%($\upa$)& 3.8\%($\upa$)& 1\%($\upa$)& 1\%($\upa$)\\
         \midrule
         ELSC-K(P) & 0.86 & 0.76 & 0.25 & 0.39 & 0.49 & 0.70 & 0.54 & 0.79 &  0.72 &  0.80 & 0.95 & 0.96 \\
         HLSC-K(P) & \bf0.87 & \bf0.78 & \bf 0.27 & \bf 0.41 & \bf 0.79 & 0.77 & 0.53 & 0.78 & 0.69 & 0.73 & \bf 0.96 & \bf 0.98\\
         \midrule
         Improvement & 1.1\%($\upa$) &2.6\%($\upa$) &8\%($\upa$) &2.5\%($\upa$) & 5.1\%($\upa$)& 10\%($\upa$)&1.9\%($\dn$) & 1.3\%($\dn$)& 4.2\%($\dn$)& 8.75\%($\dn$)& 1\%($\upa$)& 1\%($\upa$)\\
         \midrule
         %Average Gain &\color{Green}{0.02} &\color{Green}{0.02}&\color{Green}{0.02}&\color{Green}{0.02}&\color{Green}{0.15}&\color{Green}{0.04}&\color{Green}{0.01}&\color{Green}{0.01}&\color{Green}{0.00}&\color{Red}{-0.04}&\color{Green}{0.01}&\color{Green}{0.02}\\
         \toprule
         FESC(G) & 0.01 & 0.01  & 0.01 & 0.07 & 0.12 & 0.23 & 0.25 & 0.73 & 0.11 & 0.21 & 0.02 & 0.28\\
         FHSC(G)  & 0.05 & 0.02 & 0.02 & 0.04 & 0.36 & 0.52 & 0.13 & 0.39 & 0.24 & 0.48 & 0.14 & 0.67\\
         \midrule
         Improvement & 400\%($\upa$) & 100\%($\upa$) & 100\%($\upa$) & 43\% ($\upa$) & 200\%($\upa$) & 126\%($\upa$)& 48\%($\dn$) & 47\%($\dn$) & 118\%($\upa$) & 129\%($\upa$) & 600\%($\upa$)& 139\%($\upa$) \\
         \midrule
         FESC(P) & 0.01 & 0.01 & 0.03 & 0.07  & 0.11 & 0.24 & 0.16 & 0.44 & 0.11 & 0.22 & 0.10 & 0.38\\
         FHSC(P)  & 0.84 & 0.75 & 0.25 &  0.37 & 0.68 & 0.77 & 0.37 & 0.79 & 0.70 & 0.78 & 0.94 & 0.96\\
         \midrule
         Improvement & 8300\%($\upa$) & 7400\%($\upa$) & 733\%($\upa$) & 429\% ($\upa$) & 518\%($\upa$) & 221\%($\upa$)& 131\%($\upa$) & 79\%($\upa$) & 536\%($\upa$) & 254\%($\upa$) & 840\%($\upa$)& 153\%($\upa$) \\

         \midrule
         %Average Gain & & & & & & & & & & & & \\
         \toprule
         LRR &  -0.01   & 0.00 & -0.03  & 0.12  & 0.01  & 0.07  & -0.00  & 0.04  & 0.00  & 0.04  &  0.00 & 0.03\\
       
         \bottomrule
 \end{tabular}}
    \label{tab:table 3}
\end{table*}  

\begin{comment}
\begin{table*}[b]
     \caption{Extrinsic Evaluation Metrics for Datasets for the variants of FESC and LRR }
    \centering
    \scriptsize
    \resizebox{2\columnwidth}{!}{%
    \begin{tabular}{l c c |c  c | c c |cc |cc |cc }
        \toprule
        \multirow{2}{*}{\textbf{Methods}} & \multicolumn{2}{c}{\textbf{Wisconsin }} & \multicolumn{2}{c}{\textbf{Glass}} &\multicolumn{2}{c}{\textbf{Zoo}} &\multicolumn{2}{c}{\textbf{2d-20c-no0}} &\multicolumn{2}{c}{\textbf{st900}} &\multicolumn{2}{c}{\textbf{D31}} \\
         \cmidrule{2-13}

\end{tabular}%
 }
    \label{tab:table 4}
\end{table*}           
\end{comment}
\subsection{Effectiveness of HSCA:} We now assess the performance of our proposed HSCA against the following well-known versions of the Spectral Clustering algorithms [with the Gaussian (G) and Poisson (P) kernels]: 
\begin{enumerate}
    \item Euclidean Spectral Clustering Algorithm [ESCA] \cite{tutorial}
    \item Fast Euclidean Spectral Clustering Algorithm \cite{he} [we have also developed a hyperbolic version of FESC, which we named as Fast Hyperbolic Spectral Clustering Algorithms (FHSC) analogously to the HSCA and included in our comparison]
    \item Low Rank Representation Clustering (LRR) \cite{liu}
    \item Approximated spectral clustering with $k$-means-based landmark selection (ELSC-K) \cite{chencai} [we have also compared it with the respective hyperbolic variants HLSC-K with respect to the Gaussian(G) and Poisson Kernels(P)]. See Supplementary for a pseudo-code of the algorithm.
\end{enumerate}

\subsection{Discussions:}% While discussing the results given in Tables 2 and 3, let us interpret the changes in the evaluation metrics that we obtained by running the variants of the HSCA as mentioned. In Table 2, we have enlisted the results obtained by running the variants of ESCA and their corresponding hyperbolic versions. We will treat the Airport dataset separately since no ground cluster levels are available for it. Moreover Airport is a categorical dataset, which inherits a hierarchical structure. Whenever there is a hierarchy present in the dataset, the initial transformation of HSCA embeds the dataset in the similar hierarchy closer compared to the other points which appear in the different hierarchies. And therefore, the hyperbolic distance between the points in similar hierarchy become much smaller compared to the distances between points present in different hierarchy, this is exactly the place where the Euclidean Spaces exhibits a poor performance. Once the embedding is done, the rest of HSCA do classify the points almost according to their hierarchy.  
%When analyzing the outcomes presented in Tables 2 and 3, let's interpret the alterations in the evaluation metrics resulting from running the variants of HSCA as specified. In Table 2, we have outlined the outcomes from executing the variants of ESCA alongside their corresponding hyperbolic versions. We will handle the Airport dataset separately since there are no ground cluster levels available for it. Furthermore, the Airport dataset is categorical, inheriting a hierarchical structure. When a dataset includes a hierarchy, the initial transformation of HSCA embeds the dataset closer to a similar hierarchy compared to other points appearing in different hierarchies. Consequently, the hyperbolic distance between points in similar hierarchies becomes significantly smaller compared to distances between points in different hierarchies, where Euclidean Spaces exhibit poor performance. Once the embedding process is complete, the remainder of HSCA categorizes the points almost according to their hierarchy.

When analyzing the outcomes presented in Tables 2 and 3, let's delve into the alterations in the evaluation metrics resulting from running the variants of the HSCA as specified. In Table 2, we've outlined the outcomes from executing the variants of ESCA alongside their corresponding hyperbolic versions. We will handle the Airport dataset separately since there are no ground cluster levels available for it. Furthermore, the Airport dataset is categorical, inheriting a hierarchical structure. When a dataset includes a hierarchy, the initial transformation of HSCA embeds the dataset closer to a similar hierarchy compared to other points appearing in different hierarchies. Consequently, the hyperbolic distance between points in similar hierarchies becomes significantly smaller compared to distances between points in different hierarchies, where Euclidean Spaces exhibit poor performance. Once the embedding process is complete, the remainder of HSCA categorizes the points almost according to their hierarchy.

\subsubsection{Airport Dataset:} If we compare the obtained numerical results with the visual clusters depicted through the t-SNE plot in Fig. \ref{fig:4}, we observe that for the Airport dataset, the Silhoutte Score have been significantly decreased in the HSCA compared to the ESCA. This is primarily due to the fact that both the ESCA(G) and ESCA(P) have given one predominant clusters, whereas the HSCA clusters have given $5$ distinct clusters. Since the Silhoutte Score increases if the mean distance between clusters in the same dataset decreases and the mean smallest distance from the other clusters increases, the ESCA clusters are expected to give higher Silhoutte Score compared to the HSCA clusters simply because there are more distinct clusters formed in HSCA. On the other hand, the Davies-Bouldin score depends on the distances between the centroids of the same cluster vs the distances between cluster centroids in distinct clusters. Unlike the Silhoutte Score, there is a drastic change between the values in the Davies-Bouldin Scores compared to the Euclidean with respect to the Hyperbolic, formation of more distinct clusters along with the incorporated the hyperbolic metric to compute the particluar score have resulted in the drastic change. A good cluster is generally charecterized by the higher Silhoutte Score (closer to 1), lower Davies-Bouldin Score (preferably lower than 1) but high Calinski-Harabasz Score. Aparently HSCA performs poorly compared to the ESCA if the first two metrics are considered, but the Calinski-Harabasz Indices have gone significantly high in the case of HSCA, this is again due to incorporating the hyperbolic metric and the nature of distribution that the index follows. Since the ground level clusters are not available for this dataset and it has some inherent hierarchical structure, we can conclude the HSCA has surmounted the ESCA  by giving more distinct clusters.

Upon scrutinizing the numerical outcomes alongside the visual representations of clusters portrayed in the t-SNE plot in Fig. \ref{fig:4}, a notable trend emerges within the Airport dataset: the Silhouette Score experiences a significant decline in the HSCA compared to the ESCA. This phenomenon can be attributed to the fact that both the ESCA(G) and ESCA(P) methods yield one predominant cluster, whereas the HSCA approach generates five distinct clusters. Given that the Silhouette Score is contingent upon a reduction in the mean distance between clusters within the same dataset and an increase in the mean smallest distance from other clusters, it follows that ESCA clusters are predisposed to higher Silhouette Scores relative to HSCA clusters due to the greater number of distinct clusters formed by HSCA. Conversely, the Davies-Bouldin score hinges on the distances between centroids within the same cluster versus the distances between cluster centroids in distinct clusters. In contrast to the Silhouette Score, there is a notable disparity in the Davies-Bouldin Scores between the Euclidean and Hyperbolic scenarios, resulting from the formation of more distinct clusters and the incorporation of the hyperbolic metric in computing the score. A desirable cluster is typically characterized by a higher Silhouette Score (closer to 1), a lower Davies-Bouldin Score (preferably below 1), and a high Calinski-Harabasz Score. Although HSCA appears to underperform compared to ESCA based on the first two metrics, the Calinski-Harabasz Indices show a significant increase in the case of HSCA. This improvement is once again attributed to the incorporation of the hyperbolic metric and the distribution nature that the index follows. Despite the unavailability of ground-level clusters for this dataset and its inherent hierarchical structure, we can infer that HSCA outperforms ESCA by yielding more distinct clusters.

\subsubsection{Other Real Datasets: Wisconsin, Glass \& Zoo} %Here we discuss the effects of HSCA vs ESCA for the Wisconsin Breast Cancer, Glass, and Zoo datasets. In all these three cases HSCA variants are performing better [see \ref{fig:5} and supplementary] compared to the ESCA clusters. This is primarily because the datasets do have some kind of hierarchical structure and when we are embedding them into Euclidean Spaces, there is a poor cluster formations. Only the Approximated Spectral Clustering with $k-$means based landmark selection(ELSC-K) works better compared to the usual Euclidean Spectral Clustering, this is because the $k-$means based landmark selection captures the inbuilt hierarchy to some extent and the subsequent spectral clusterings provide better results. But again, the Hyperbolic variants of the ELSC-K are still working better simply because the hyperbolic space is a better candidate for embedding hierarchical data. 

Here we examine the effects of HSCA versus ESCA for the Wisconsin Breast Cancer, Glass, and Zoo datasets. In all these three cases, HSCA variants outperform ESCA clusters [see \ref{fig:5} and the supplementary document]. This is primarily attributed to the datasets possessing some form of hierarchical structure. When embedding them into Euclidean Spaces, poor cluster formations arise. Notably, only the Approximated Spectral Clustering with $k$-means based landmark selection (ELSC-K) exhibits superior performance compared to standard Euclidean Spectral Clustering. This can be attributed to the $k$-means based landmark selection, which captures the inherent hierarchy to some extent, leading to improved spectral clusterings. However, Hyperbolic variants of ELSC-K still outperform, as the hyperbolic space is better suited for embedding hierarchical data.

\subsubsection{Synthetic Datasets: 2d-20c-no0, st900 \& D15} The similar clustering forms also work for these datasets. The Poincar\'{e} embedding takes care of the part that the points in the same hierarchy will remain close. Here we also observe that the variants of HSCA perform way better than the corresponding variants of the ESCA as expected. The ESCA and HSCA clusters formed on the st-900 dataset have been shown in Figure \ref{fig:6}. For other synthetic datasets, we refer to the supplementary. 

\subsection{Ablation Experiments} The kernel hyperparameter $\sigma$ requires tuning to achieve optimal performance. In our experiments, we set $\sigma$ to a small value (approximately $0.1$), which has been found to be optimal for most variants of the HSCA. We have observed that there is minimal fluctuation in the ARi or NMi values when the hyperparameter values ($\frac{1}{\sigma^2}$ for the Gaussian Kernel and $\frac{1}{2\sigma}$ for the Poisson Kernel) are very small (less than $0.1$). These values exhibit bounded variation within a small range when the hyperparameter is very small. Our approach for this analysis follows a similar methodology to that presented in \cite{bai}. For detailed graphs illustrating the tuning process, please refer to Figures \ref{fig:7} and \ref{fig:8}.

The clusters related to other datasets can be found in Section \emph{Experiments and Results} of the supplementary document. 

\begin{figure*}[t]
     \centering
     \begin{subfigure}[b]{0.16\textwidth}
         \centering
         \includegraphics[width=\textwidth]{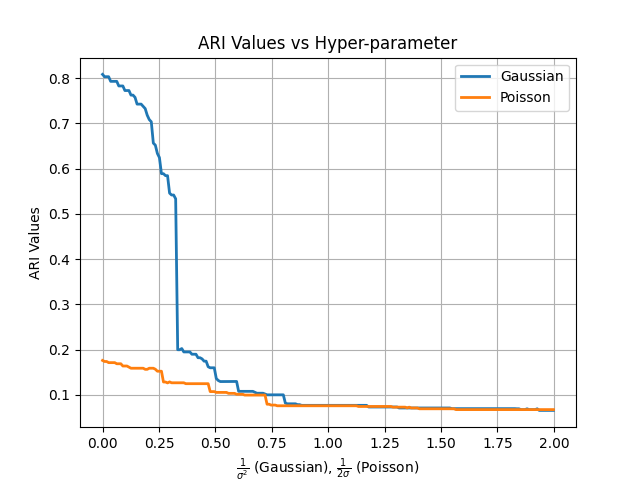}
         \caption{WBCD}
         \label{fig:ari_wisc1}
     \end{subfigure}
     \hfill
     \begin{subfigure}[b]{0.16\textwidth}
         \centering
         \includegraphics[width=\textwidth]{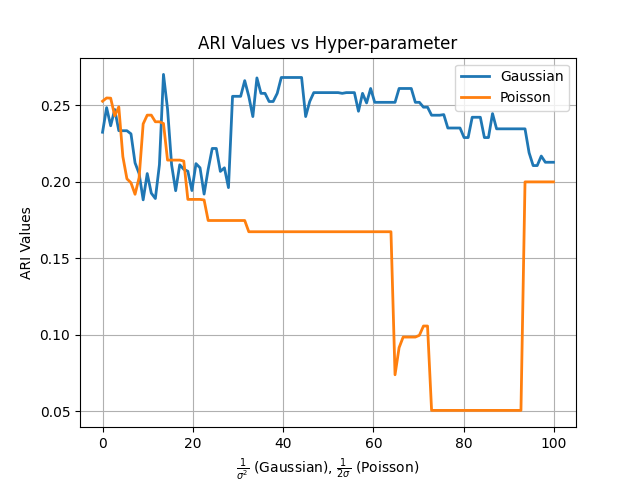}
         \caption{Glass}
         \label{fig: ari_glass1}
     \end{subfigure}
     \hfill     
     \begin{subfigure}[b]{0.16\textwidth}
         \centering
         \includegraphics[width=\textwidth]{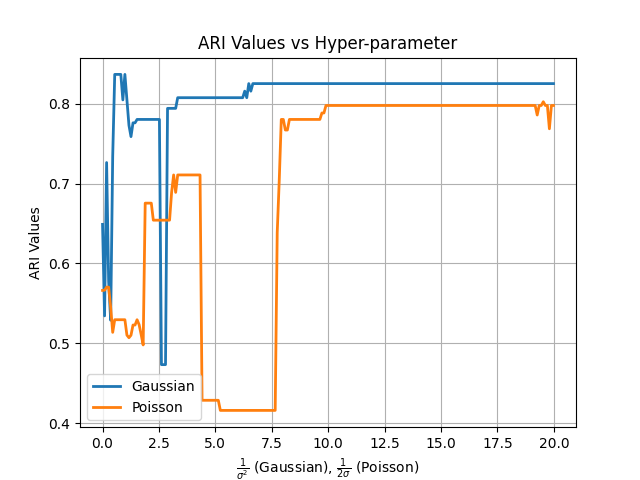}
         \caption{Zoo}
         \label{fig:ari_zoo1}
     \end{subfigure}
     \hfill
     \begin{subfigure}[b]{0.16\textwidth}
         \centering
         \includegraphics[width=\textwidth]{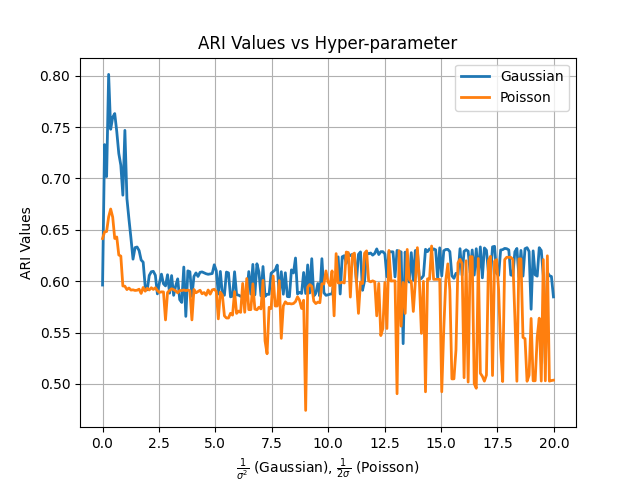}
         \caption{2d-20c-no0}
         \label{fig: ari_2d-20c-no01}
     \end{subfigure}
     \hfill     
     \begin{subfigure}[b]{0.16\textwidth}
         \centering
         \includegraphics[width=\textwidth]{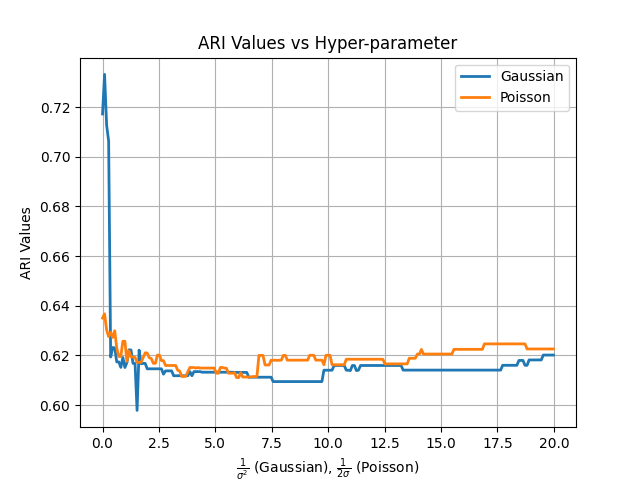}
         \caption{st900}
         \label{fig: ari_st9001}
     \end{subfigure}
     \hfill
     \begin{subfigure}[b]{0.16\textwidth}
         \centering
         \includegraphics[width=\textwidth]{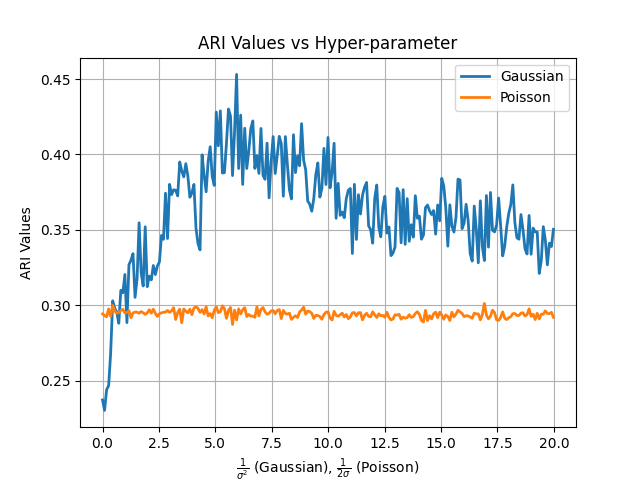}
         \caption{D31}
         \label{fig:ari_D311}
     \end{subfigure}
     \caption{Visualization of ARI Values with respect to the hyperparameter $   \frac{1}{\sigma^2}$ (for Gaussian) and $\frac{1}{2\sigma}$ (for Poisson)}
     \label{fig:7} 
\end{figure*}
%\begin{comment}
\begin{figure*}[]
     \centering
     \begin{subfigure}[b]{0.16\textwidth}
         \centering
         \includegraphics[width=\textwidth]{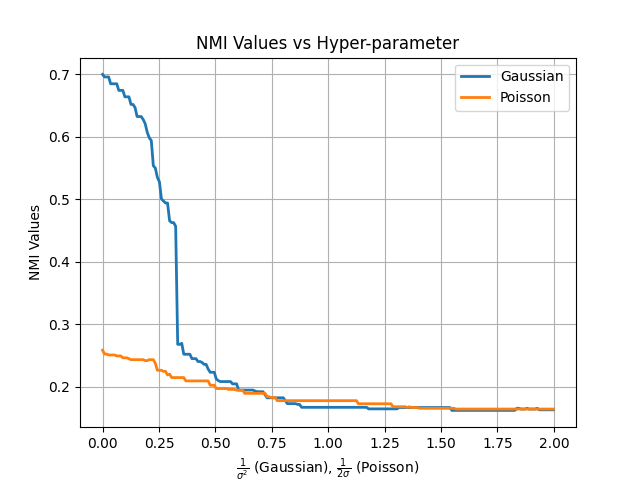}
         \caption{WBCD}
         \label{fig:nmi_wisc1}
     \end{subfigure}
     \hfill
     \begin{subfigure}[b]{0.16\textwidth}
         \centering
         \includegraphics[width=\textwidth]{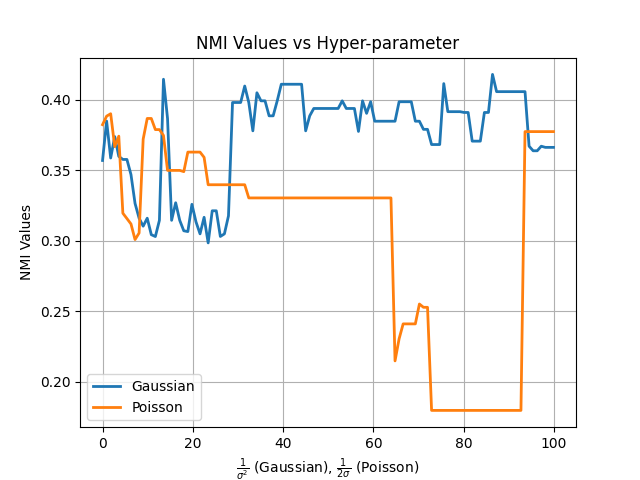}
         \caption{Glass}
         \label{fig: nmi_glass1}
     \end{subfigure}
    \hfill
     \begin{subfigure}[b]{0.16\textwidth}
         \centering
         \includegraphics[width=\textwidth]{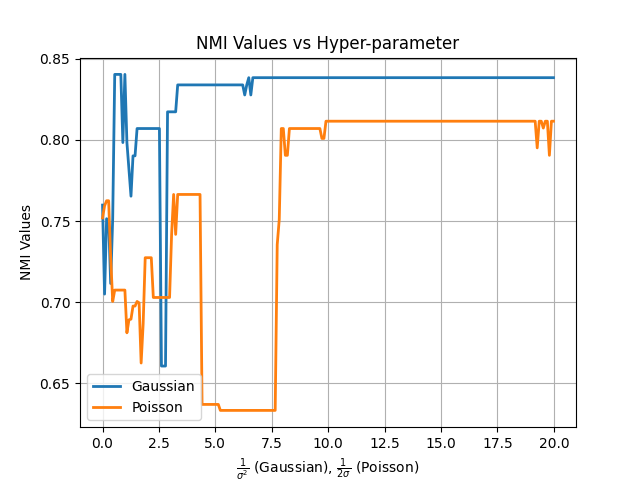}
         \caption{Zoo}
         \label{fig:nmi_zoo1}
     \end{subfigure}
     \hfill
     \begin{subfigure}[b]{0.16\textwidth}
         \centering
         \includegraphics[width=\textwidth]{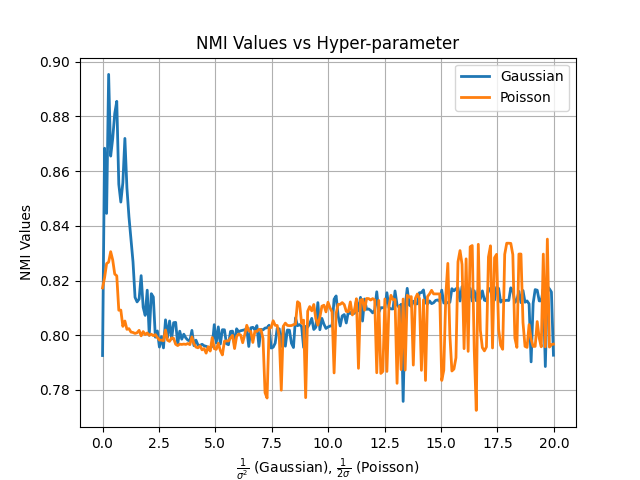}
         \caption{2d-20c-no0}
         \label{fig: nmi_2d-20c-no01}
     \end{subfigure}
     \hfill
     \begin{subfigure}[b]{0.16\textwidth}
         \centering
         \includegraphics[width=\textwidth]{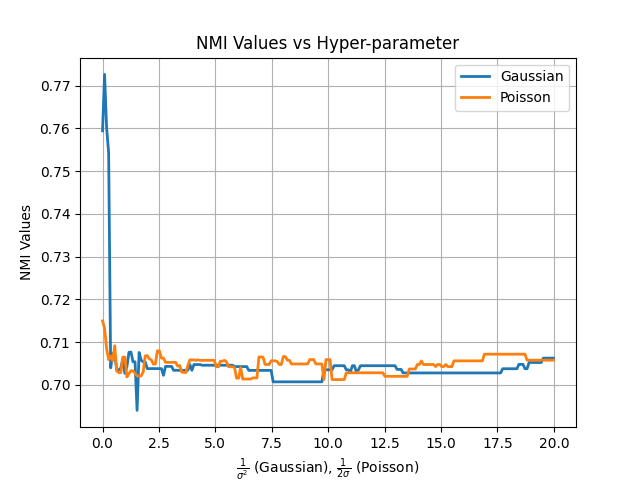}
         \caption{st900}
         \label{fig: nmi_st9001}
     \end{subfigure}
     \hfill
     \begin{subfigure}[b]{0.16\textwidth}
         \centering
         \includegraphics[width=\textwidth]{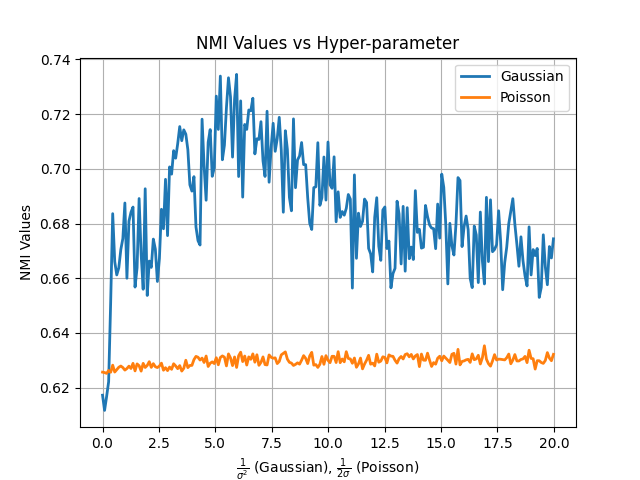}
         \caption{D31}
         \label{fig:nmi_D311}
     \end{subfigure}
     \caption{Visualization of NMI Values with respect to the hyperparameter $\frac{1}{\sigma^2}$ (for Gaussian) and $\frac{1}{2\sigma}$ (for Poisson)}
    \label{fig:8}
\end{figure*}
%\end{comment}

\section{Conclusion}\label{sec:7}
     In this paper, we highlight the limitations of Euclidean space in conveying meaningful information and representing arbitrary tree or graph-like structures. We emphasize the inability of Euclidean spaces with low dimensions to embed such structures while preserving the associated metric. Conversely, Hyperbolic Spaces offer a compelling solution, even in shallow dimensions, to efficiently represent such data and yield superior clustering results compared to Euclidean Spaces. Our primary contributions include proposing a spectral clustering algorithm tailored for hyperbolic spaces, where a hyperbolic similarity matrix replaces the traditional Euclidean one. Additionally, we provide theoretical analysis on the weak consistency of the algorithm, showing convergence at least as fast as spectral clustering on Euclidean spaces. We also introduce hyperbolic versions of well-known Euclidean spectral clustering variants, such as Fast Spectral Clustering with approximate eigenvectors (FastESC) and Approximate Spectral Clustering with $k$-means-based Landmark Selection.

Our approach to clustering data in hyperbolic space via spectral decomposition of the modified similarity matrix reveals a more effective method for clustering complex hierarchical data compared to Euclidean Spaces. Notably, this algorithm converges at least as rapidly as Spectral Clustering on Euclidean Spaces without introducing additional computational complexity. Moreover, variants of Hyperbolic Spectral Clustering exhibit superior adequacy and performance over their Euclidean counterparts. We anticipate that this method will expedite significant advancements in non-deep machine learning algorithms.

\section*{Availability of Datasets and Codes:} The airport dataset is available on \href{https://www.kaggle.com/datasets/tylerx/flights-and-airports-data}{https://www.kaggle.com/datasets/tylerx/flights-and-airports-data}. The other real datasets have been taken from \href{https://github.com/milaan9/Clustering-Datasets/tree/master/01.%20UCI}{https://github.com/milaan9/Clustering-Datasets/tree/master/01.\%20UCI}  and the synthetic datasets are available \href{https://github.com/milaan9/Clustering-Datasets/tree/master/02.%20Synthetic}{https://github.com/milaan9/Clustering-Datasets/tree/master/02.\%20Synthetic}. %The Pyhton programs pertaining to this paper are available at \href{https://github.com/sagarghosh1729/HSCA}{https://github.com/sagarghosh1729/HSCA}.  

\section{Appendix}
\begin{comment}
    \section{Proof of Proposition in Section 3}
\subsection{Proof of Proposition \ref{prop:3.1}}
\begin{proof}
    Let $\{U_i, \phi_i\}_{i\in I}$ be charts on $M$ such that $\{U_i\}$ is a locally finite cover of $M$. Let $\{\psi_i\}_{i\in I}$ be a partitions of unity subordinate to the open cover $\{U_i\}$. By Whitney Embedding Theorem every smooth manifold of dimension $n$ admits an embedding in $\rr^{2n}$. Let $g_E$ be the Euclidean Metric on $\rr^{2n}$. Define $g_M:=\sum_{i\in I}\psi_i \phi^*_i g_E$, where $\phi^*_i g_E$ is the pullback of the Euclidean Metric by the charts $\{\phi_i\}_{i\in I}$. Then $Supp(\psi_i \phi^*_i g_E)\subseteq U$, $\psi_i \phi^*_i g_E$ is smooth on $M$ and the sum is defined as $\{U_i\}$ is locally finite. Hence $g_M:M\to  T^*M^{\otimes 2}$  is $C^\infty$. For any $p\in M$, there is $i\in I$ such that $\psi_i(p)>0$. Hence, ${g_M}_p(v,v)\geq \psi_i(p)\phi^*_i g_E(v,v)\geq 0$, for all $ v \in T_p(M)$. Hence $g_M$ is a Riemannian Metric on $M$.
\end{proof}  
\end{comment}

\section{Proofs of Lemmas and Theorem from Section 5}
\subsection{Proof of Lemma \ref{lem:5.1}}
\begin{proof}
    \textbf{Step 1:} We have $\|x\|,\|y\|< 1$, hence $\|x\|^2,\|y\|^2<1\implies (1-\|x\|^2)(1-\|y\|^2)<1$. Hence we can write 
    \begin{align*}
        \frac{\|x-y\|^2}{(1-\|x\|^2)(1-\|y\|^2)}>\|x-y\|^2.
    \end{align*}
    But we also have 
    \begin{align*}
        \delta(x,y)=2\frac{\|x-y\|^2}{(1-\|x\|^2)(1-\|y\|^2)}.
    \end{align*}
    Combining this with the last inequality, we get
    \begin{align*}
        \delta(x,y)>2\|x-y\|^2.
    \end{align*} \\

\textbf{Step 2:} For $x\in\rr$, $\frac{d}{dx}(\sinh^{-1}(x))=\frac{1}{\sqrt{1+x^2}}>0$. Therefore the inverse sine hyperbolic function is a strictly increasing function of $x$.  By Step 1, we have $\delta(x,y)>2\|x-y\|^2$. Therefore, we have $\frac{\delta(x,y)}{2}\geq \|x-y\|^2$. This also implies $\\sqrt{\frac{\delta(x,y)}{2}}\geq\|x-y\|$. \\
Since $d(x,y)=2sinh^{-1}\left(\sqrt{\frac{\delta(x,y)}{2}}\right)$ and the inverse sine hyperbolic function is increasing, we can write 
\begin{align*}
    d(x,y)\geq 2\sinh^{-1}(\|x-y\|)
\end{align*}
We know that for $0<s<t$, $exp(-s)>exp(-t)$. This enables us to write 
\begin{align*}
    K_{H_G}(x,y)&=exp(-ad(x,y)^2)\\
    &\leq exp(-4a[sinh^{-1}(\|x-y\|)]^2).
\end{align*}

\textbf{Step 3:} Note that for $0\leq x\leq 1$, $\frac{1}{\sqrt{1+x^2}}\geq \frac{1}{\sqrt{2}}$.  Let $f(x):=\sinh^{-1}(x)-\frac{x}{2}$. Then $f$ is differentiable and we get $f^\prime(x)=\frac{1}{\sqrt{1+x^2}}-\frac{1}{2}\geq \frac{2-\sqrt{2}}{2\sqrt{2}}$. Therefore $f$ is increasing on $[0,1]$ and for $0\leq x\leq 1$, $\sinh^{-1}(x)\geq \frac{x}{2}$. Hence $\exp(-\sinh^{-1}(\|x-y\|^2))\leq \exp\left(-\frac{\|x-y\|^2}{2}\right)$. Therefore following step 2, we get
\begin{align*}
    K_{H_G}(x,y)&\leq \exp(-4a[\sinh^{-1}(\|x-y\|)]^2)\\
    &\leq \exp\left(-4a\frac{\|x-y\|^2}{4}\right)\\
    &=\exp(-a\|x-y\|^2)=K(x,y),
\end{align*} 
\end{proof}

\subsection{Proof of Lemma \ref{lem:5.2}}
\begin{proof}
    $K_{H_G}(x)=K_H(x,0)\leq K(x,0)=\exp(-a\|x\|^2)$ [by Lemma \ref{lem:5.1} and Remark \ref{rem:5.2}].
Therefore following step 3 of Lemma \ref{lem:5.1} we write,
\begin{align*}
    \int_{H}|K_{H_G}(x)|dx &\leq \int_{H}|\exp(-a\|x\|^2)|dx\\
    &\int_{H}\exp(-a\|x\|^2)dx\\ 
    & \leq \int_{\rr^n}\exp(-a\|x\|^2)dx \\
    & <\infty
\end{align*}  as $H$ is only a subset of the unit ball in $\rr^n$. The last integral is finite since the integrand is the usual Gaussian distribution. 
\end{proof}
\subsection{Proof of Lemma \ref{lem:5.3}}
\begin{proof}
    $f$ is radial if and only if for every $M\in SO_n(\rr^n)$ [where $SO_n(\rr^n)$ is the special unitary group on $\rr^n$, i.e. consisting of all $n\times n$ matrices over $\rr$ with determinant $1$], $f(Mx)=f(x)$ [as the operation $x\to Mx$ only rotates $x$, does not change its magnitude, i.e. $\|Mx\|=\|x\|$]. Then for any arbitrary $M\in SO_n(\rr^n)$
\begin{align*}
    \widehat{f}(Mt) &=\int_{\Omega}f(x)e^{-i<Mt,x>}dx\\
    &=\int_{M(\Omega)} f(Ms)e^{-i<Mt,Ms>}ds\hspace{2ex}\\
    &[\text{change of variable $x\to Ms$}]\\
    &= \int_{\Omega}f(s)e^{-i<t,s>}ds \hspace{2ex}[\text{since $\Omega$ is symmetric}]\\
    &= \widehat{f}(t),
\end{align*}
where the second equality follows from the fact that we get by the conjugate linearity of the inner product: $<Mt,Ms>=\\<M^*Mt,s>$  $=<t,s>$ since $M^*M=I_n$ [$M\in SO_n(\rr^n)$]. 
\end{proof}
\subsection{Proof of Lemma \ref{lem:5.4}}
\begin{proof}
     Let $f(x)=K_{H_G}(x)=exp(-ad(x,0)^2)$. Then by Lemma \ref{lem:5.1}, we have $f(x)\leq exp(-a\|x\|^2)$ for all $x\in H$. Exploiting fully the fact that $\widehat{k}$ is radial (and hence real valued), we get
\begin{align*}
    |\widehat{K}(w)| &=|\int_{H}f(x)e^{-iw^tx}dx|\\
    &=|\int_{H}f(x)e^{a\|x\|^2}e^{-a\|x\|^2}e^{-iw^tx}dx|\\
    & \leq \int_{H}|f(x)e^{a\|x\|^2}e^{-a\|x\|^2}e^{-iw^tx}|dx\\
    &\leq\int_H |e^{-a\|x\|^2}e^{-iw^tx}|dx \hspace{2ex}\\
    &[\text{This is just Fourier Transform of the Euclidean}\\
    &\text{Gaussian Kernel over $H$}]\\
    &\leq \int_{\rr^n} |e^{-a\|x\|^2}e^{-iw^tx}|dx\\
    &\leq C^\prime e^{-p\|w\|^2}\\
    &\leq Cexp(-l\|w\|),
\end{align*}
where $C^\prime$ and $C$ are some appropriately chosen constants. 
\end{proof}
\subsection{Proof of Theorem \ref{them:5.1}}
\begin{proof}
    Combining Lemma \ref{lem:5.4} and Theorem 3 \cite{zhou} we get
    \begin{align*}
        \log(\mathcal{N}(\mathcal{F},\epsilon, \|\cdot\|_{\infty}))\leq C_0 \log\left(\frac{1}{\epsilon}\right)^{d+1},
    \end{align*}
    for some constant $C_0$ chosen appropriately and $d$ is the dimension of $H$. Since $d$ is a constant for $H$, we can write the above inequality as 
    \begin{align*}
        \log(\mathcal{N}(\mathcal{F},\epsilon, \|\cdot\|_{\infty}))\leq C_1 \log\left(\frac{1}{\epsilon}\right)^{2}.
    \end{align*}
    Following the same sequence of computation as in Theorem 19\cite{lux}, we get
    \begin{align*}
        \int_{0}^{\infty}\sqrt{\log(\mathcal{N}, \epsilon, L^2(P_n))}d\epsilon < \infty
    \end{align*}
Hence following Theorem 19\cite{lux} we write
\begin{align*}
        \sup_{f\in \mathcal{F}}|P_nf-Pf|&\leq \frac{c}{\sqrt{n}}\int_{0}^{\infty}\sqrt{\log(\mathcal{N}, \epsilon, L^2(P_n))}d\epsilon\\
        &+\sqrt{\frac{1}{2n}\log\left(\frac{2}{\delta}\right)}\\
        &<\frac{C_1}{\sqrt{n}}+\sqrt{\frac{1}{2n}\log\left(\frac{2}{\delta}\right)},
    \end{align*}
for some appropriately chosen constant $C_1$. Since $\delta>0$ we get,
\begin{align*}
    \sup_{f\in \mathcal{F}}|P_nf-Pf| \leq C\left(\frac{1}{\sqrt{n}}\right).
\end{align*}
Finally Finally, combining theorem 16 of \cite{lux} with the last inequality, we get
\begin{align*}
    \sup_{f\in \mathcal{F}}|P_nf-Pf|= \mathcal{O}\left(\frac{1}{\sqrt{n}}\right).
\end{align*}
\end{proof}

\subsection*{System Specification:} 
The experiments were carried out on a personal computer with $12$th Gen Intel(R) Core(TM) i5-1230U   1.00 GHz Processor, 16 GB RAM, Windows $11$ Home $22$H$2$ and Python $3.11.5$.

\end{document}